\documentclass[letterpaper,10pt]{article}

\usepackage[margin=1.28in]{geometry}
\usepackage{natbib}
\usepackage{newtxtext}
\usepackage{titling}
\usepackage{authblk}
\usepackage[final]{microtype}
\usepackage{graphicx}
\usepackage{booktabs} % for professional tables
\usepackage{float}
\usepackage{hyperref}

\usepackage[table]{xcolor} % changed from color to xcolor (2021/11/24)

\usepackage{amsmath}
\usepackage{amssymb}
\usepackage{mathtools}
\usepackage{amsthm}
\usepackage{bbm}
\usepackage{complexity}

\usepackage[noabbrev, capitalize]{cleveref}
\usepackage{multirow}
\usepackage{wrapfig}
\usepackage{subcaption}
\usepackage[inline]{enumitem}

\usepackage[color=lightgray!20,textsize=scriptsize]{todonotes}

\setcitestyle{authoryear}
\hypersetup{
    colorlinks,
    linkcolor={red!50!black},
    citecolor={cyan!50!black},
    urlcolor={cyan!70!black}
}
\definecolor{trueblue}{rgb}{0.0, 0.45, 0.81}

\newcommand{\base}{\textbf{\textcolor{teal}{base}}}
\newcommand{\random}{\textbf{\textcolor{red!70!black}{random}}}
\DeclareRobustCommand{\stirling}{\genfrac\{\}{0pt}{}}

\graphicspath{{./figures/}{../figures/}{../main/}{./icml24/}{../icml24/}}

\makeatletter
\renewcommand\AB@affilsepx{, \protect\Affilfont}

\renewcommand\Affilfont{\normalfont\fontsize{9.5}{14.4}\selectfont}
\makeatother

\title{Compositional Capabilities of Autoregressive Transformers:\\
    A Study on Synthetic, Interpretable Tasks
}

\author[1]{Rahul Ramesh}
\author[2,3]{Ekdeep Singh Lubana}
\author[4]{Mikail Khona}
\author[3]{\protect\\Robert P. Dick}
\author[2,5]{Hidenori Tanaka}
\affil[1]{Computer and Information Science, University of Pennsylvania}
\affil[2]{Center for Brain Science, Harvard University}
\affil[3]{Electrical Engineering and Computer Science, University of Michigan}
\affil[4]{Physics, MIT}
\affil[5]{Physics \& Information Labotratories, NTT Research}
\date{}

\theoremstyle{plain}
\newtheorem{theorem}{Theorem}[section]

\theoremstyle{definition}
\newtheorem{definition}[theorem]{Definition}

\newtheorem{conjecture}[theorem]{Conjecture}
\theoremstyle{remark}

\Crefname{Figure}{Fig.}{Figs.}
\crefname{figure}{fig.}{figs.}
\Crefname{Section}{Section}{Sections}
\crefname{section}{section}{sections}

\Crefname{Appendix}{Appendix}{Appendices}
\crefname{appendix}{appendix}{appendices}

\begin{document}

\maketitle
\vspace*{-4em}

\begin{abstract}
    Transformers trained on huge text corpora exhibit a remarkable set of capabilities, e.g., performing basic arithmetic. Given the inherent compositional nature of language, one can expect the model to learn to compose these capabilities, potentially yielding a \textit{combinatorial explosion} of what operations it can perform on an input. Motivated by the above, we train autoregressive Transformer models on a synthetic data-generating process that involves compositions of a set of well-defined monolithic capabilities. Through a series of extensive and systematic experiments on this data-generating process, we show that:
    \begin{enumerate*}
        \item autoregressive Transformers can learn compositional structures from small amounts of training data and generalize to exponentially or even combinatorially many functions;
        \item generating intermediate outputs when composing functions is more effective for generalizing to new, unseen compositions than not generating any intermediate outputs
        \item biases in the order of the compositions in the training data result in Transformers that fail to compose some combinations of functions; and
        \item the attention layers select which capability to apply while the feed-forward layers execute the selected capability.
    \end{enumerate*}
\end{abstract}

\section{Introduction}
Large scale Transformers pretrained on huge text corpora have revolutionized machine learning in recent years~\citep{radford2018improving, radford2019language, brown2020language, sanh2021multitask, wei2021finetuned, thoppilan2022lamda, touvron2023llama}. 
Due to an ever-increasing interest in adopting these models in our daily lives, evaluating and predicting their capabilities has become increasingly important~\citep{bommasani2021opportunities, ganguli2022predictability, shevlane2023model, rae2021scaling, hoffmann2022training, tay2022transcending, henighan2020scaling, hernandez2021scaling, sharma2020neural}. 
Motivated by this, recent works have performed extensive empirical analyses to understand the possibilities and limitations of using these models in practical tasks of interest. 
\begin{figure}[!tb]
    \centering
    \includegraphics[width=0.99\linewidth]{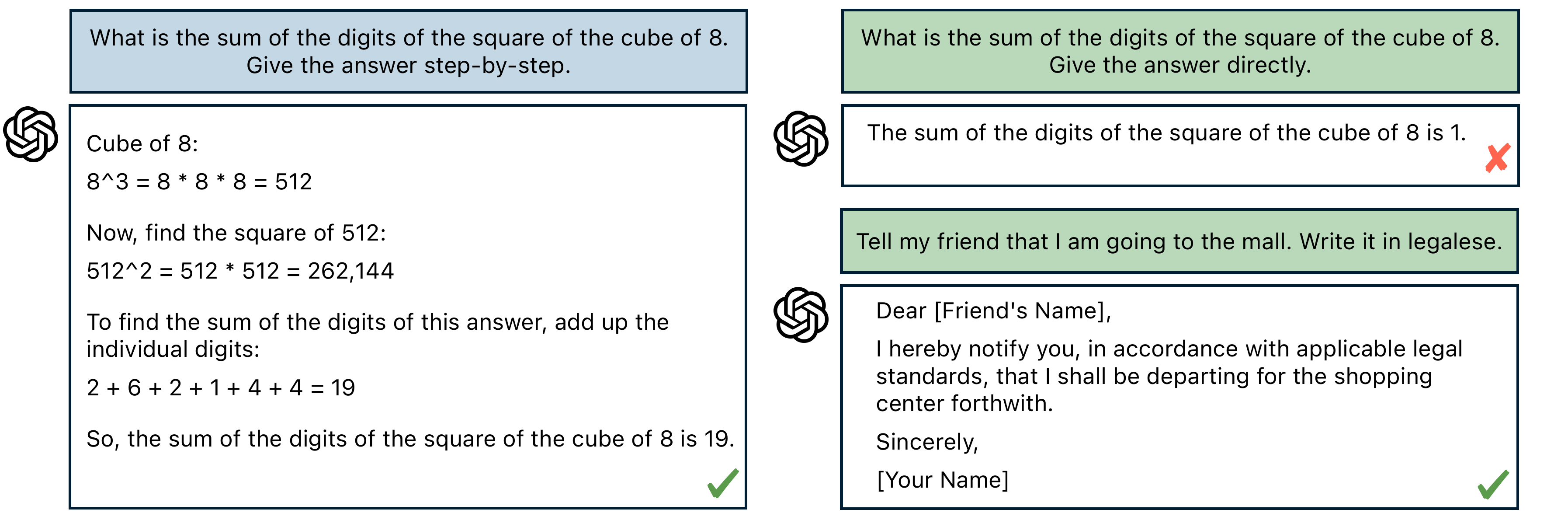}
    \caption{
        \textbf{Signatures of compositionality.}
        ChatGPT~\citep{bubeck2023sparks} correctly responds to prompts that require composition of atomic arithmetic capabilities (sum, cube, square)---we argue these prompts are unlikely to be in the training data. However, the model does not always compose reliably (top-right panel). This motivates us to study the extent to which a Transformer can learn to compose its capabilities by mere pretraining on a compositional domain.
    }
    \label{fig:llm_example}
\end{figure}
For example, such works show large language models (LLMs) can generate coherent text completions based on a provided context, perform code generation and debugging, use online APIs and tools in an automated manner, and even solve multimodal problems such as image captioning~\citep{wei2022emergent, bubeck2023sparks, austin2021program, chen2021evaluating, lee2023holistic, liang2022holistic, qin2023toolllm, liu2023visual,  suzgun2022challenging, srivastava2022beyond}. 
While such benchmarking of pretrained models is extremely valuable, it often focuses on evaluating rather ``narrow'' or ``atomic'' capabilities; for example, the ability to identify whether a given passage of text is biased or toxic~\citep{gehman2020realtoxicityprompts, liang2022holistic}. 
However, given the compositional nature of training data (such as language), a model could learn to \textit{compose} its atomic capabilities and perform complex tasks that it was never explicitly trained for.
This can lead to an underestimation of the capabilities of the model; vice versa, if the model does not learn to compose, we can be certain that benchmarking for atomic capabilities is sufficient to characterize the model.
% 

% \textbf{This paper.} 
Motivated by the above, we analyze if a Transformer trained on a compositional data-generating process, without any special modifications to the usual training pipeline, can learn \textit{both} relevant atomic capabilities and an ability to compose those capabilities. 
\citet{bubeck2023sparks} recently show that LLMs exhibit ``sparks'' of such
compositionality, e.g., generating text that merges content of varying styles or
evaluate mathematical expressions through the application of a sequence of functions (\Cref{fig:llm_example}). 
However, due to their black-box nature, it is unclear if an LLM actually learns to compose capabilities or merely memorizes relevant samples from its training data. 
Moreover, while interacting with an LLM, it can be difficult to guarantee that we are utilizing a prompt that will appropriately guide the model to use the capabilities we desire, let alone compose them. 
% Correspondingly, we may end up claiming the model lacks a certain capability, when in fact we may not be utilizing the appropriate context for eliciting it~\citep{suzgun2022challenging, reynolds2021prompt, lu2023emergent, wei2022chain}.

To circumvent challenges faced with LLMs pretrained on real world data and focus on our specific motivation, \textit{``can an autoregressive Transformer trained on compositional data learn to compose its capabilities''}, we choose to limit the purview of this work to a well-defined synthetic domain. 
This is similar in spirit to recent works that utilize synthetic datasets generated using objects like first-order logic machines, context-free grammars, linear regressors, modular arithmetic, and even board games to establish and understand phenomenology of modern neural networks~\citep{liu2022transformers, allen2023physics, allen2023knowledge, allen2023manipulation, garg2022can, li2023dissecting, saparov2022language, chan2022data, bhattamishra2020ability, zhou2023algorithms, nanda2023progress, nanda2023emergent, liEmergentWorldRepresentations2023, lubana2023mechanistic, jones2021scaling}. 
The goal of such works, including ours, is to develop interpretable demonstrations and mechanistic hypotheses that enable a characterization of the target phenomenology in a controlled setting. 
Accordingly, we emphasize that we do not intend to develop novel protocols for improving Transformers' ability to compositionally generalize, but rather to demonstrate its existence and understand what drives it.
Overall, we make the following contributions.
\begin{itemize}[leftmargin=16pt, itemsep=3pt, topsep=0pt, parsep=1pt, partopsep=1pt]
    \item \textbf{A minimal synthetic setup for characterizing Transformers' ability to compose.} 
        We propose a minimal setup involving compositions of predefined functions $\mathcal{F}$ (bijections and permutations) that operate on a string of arbitrary tokens (\Cref{sec:setup}), which allows us to precisely study the ability of Transformers to compose functions. Motivated by instruction induction and tuning in LLMs~\citep{honovich2022instruction, wei2021finetuned}, we instantiate a notion of ``task tokens'' which specify what functions are to be applied to the input string. This helps us avoid any ambiguity in task-specification~\citep{shah2022goal}. 

    \item \textbf{Transformers show explosion of capabilities.} 
    We characterize the ability of a Transformer trained on our proposed setup to compositionally generalize, i.e., to apply a composition of specific functions chosen from $\mathcal{F}$, to an input string. We show that a Transformer, trained on very few compositions, can generalize to exponentially or even combinatorially many functions (\Cref{ss:explosion})---these functions are entirely ``out-of-distribution'', i.e., the model never sees them in its training data and hence was not explicitly trained to learn them. Crucially, allowing the model to recursively process its intermediate outputs---i.e., stepwise inference~\citep{kojima2022large, wei2022chain}---significantly improves compositional generalization (\Cref{ss:prompting,s:app:step_vs_direct}).

    \item \textbf{Characterizing limitations and mechanisms of compositionality in a Transformer.} 
        We formalize a notion of ``distance'' between the functions seen by the model during pretraining and the ones it is evaluated on, hence enabling a precise characterization of when the model struggles to compose (\Cref{ss:ooo}). As we show, the training data determines whether the Transformer generalizes to an exponential or combinatorial set of functions---which we call in-order and and out-of-order generalization respectively. Furthermore, linear probing~\citep{tenney2019bert, liEmergentWorldRepresentations2023}, and an analysis of the attention maps suggests the following mechanism for solving our task: the attention layer selects the task token and the fully connected layers compute the function corresponding to it (\Cref{ss:mechinterp}). We also prove the existence of Transformers that can compositionally generalize to our task and analyze why stepwise inference helps with it (\Cref{s:app:theory}). Our mechanistic analysis and theoretical construction align extremely well.
    % Furthermore, by using the popular linear probing protocol used for understanding Transformer internals~\citep{tenney2019bert, liEmergentWorldRepresentations2023}, we show Attention layers in the latter half of the model play a crucial role in enabling compositional generalization in a Transformer (Sec.~\ref{ss:mechinterp}).
\end{itemize}

\section{Related Work}
\label{sec:related}

\paragraph{Capabilities in a Transformer.} Transformers pretrained on large-scale, web-crawled datasets have been shown to exhibit a slew of interesting capabilities, such as basic arithmetic, question answering, commonsense knowledge reasoning, stylistic transformation of a piece of text, and even multimodal reasoning~\citep{radford2018improving, radford2019language, brown2020language, bubeck2023sparks, wei2022emergent, wei2021finetuned, rae2021scaling, chowdhery2022palm, austin2021program, chen2021evaluating, bommasani2021opportunities}. 
However, this generality can come at the cost of a model also learning capabilities that are undesirable~\citep{bommasani2021opportunities, tamkin2021understanding, chan2023harms}, e.g., producing sensitive, biased, or toxic outputs~\citep{weidinger2021ethical, mcguffie2020radicalization, garrido2021survey, lin2021truthfulqa, jiang2021can, abid2021persistent, parrish2021bbq, xu2021detoxifying, huang2019reducing, sheng2019woman, gehman2020realtoxicityprompts, xu2020recipes, tamkin2021understanding}. 
This has motivated several works focused on understanding capabilities of a pretrained model, including (i) \textit{predicting} capabilities of a \textit{future} model, e.g., via fitting power laws to data/model scaling results~\citep{rae2021scaling, hoffmann2022training, hernandez2021scaling, sharma2020neural, arora2023theory} and (ii) \textit{eliciting} capabilities of a \textit{given} model, e.g., via identification of appropriate prompts or via step-wise inference protocols such as chain-of-thought, to understand what tasks a model can be reliably used for~\citep{liang2022holistic, suzgun2022challenging, lee2023holistic}. 
However, we argue that measuring a language model's performance on benchmarks to identify the existence of a set of capabilities is bound to be insufficient for characterizing what tasks it can perform:
given the compositional nature of data these models are trained on, it is possible that they learn to \textit{compose} capabilities, hence learning to perform several more tasks than we explicitly train them on. In fact, with a related motivation, \citet{yu2023skill} design a benchmark for evaluating a model's ability to combine its skills in a recent contemporary work.

\vspace*{-10pt}
\paragraph{Compositionality in neural networks.} The ability to compositionally reason has been touted as a cornerstone of human intelligence~\citep{fodor2002compositionality, fodor1988connectionism, fodor1975language, schulz2016probing}. 
Accordingly, several works have studied the ability of a neural network to compositionally generalize, usually demonstrating a negative result, and correspondingly developing explicit strategies that help improve the model's ability to generalize~\citep{livska2018memorize, hupkes2018learning, lake2018generalization, csordas2021neural, csordas2021devil, csordas2022ctl, ontanon2021making, lepori2023break, lewis2022does, yun2022vision, okawa2023compositional, hosseini2022compositional}. 
Our work differs from prior literature in several ways. 
First, we do not intend to develop protocols for improving compositional generalization in a Transformer; instead, we show that Transformers can learn to compose its capabilities and perform tasks it was never explicitly trained on, with autoregressive training on tokens from a compositional data-generating process. To this end, we define a synthetic task that allows for perfect task specification and which avoids ambiguity from prompt misspecification. While similar to the compositional table lookup task used in prior work~\citep{livska2018memorize, csordas2022ctl}, our task involves a much larger set of capabilities to train and test for (3125 or 4 million, depending on the setup, compared to 128 capabilities in prior work).
Second, we aim to understand the extent of compositional generalization in a Transformer trained on our proposed domain, i.e., what kind of compositions does the model fail to perform and when. We define a framework to precisely characterize these failures modes and use the popular linear probing protocol for understanding model internals to show the critical role of attention layers in enabling compositionality~\citep{liEmergentWorldRepresentations2023}.
Finally, we analyze the impact of step-wise inference protocols, wherein intermediate outputs generated by the model are recursively passed to it as inputs, and which has been used for solving several challenging benchmark tasks recently~\citep{suzgun2022challenging, wei2022chain}. 
Similar to our work, \citet{li2023dissecting} study step-wise inference in Transformers trained on synthetic data from a compositional data generating process. However, there are notable differences---we show that Transformers compositionally generalize to combinatorially many new functions and carefully controlling the training data allows us to highlight the benefit of step-wise inference. Furthermore, \citet{li2023transformers} study compositionality with prompts used for in-context learning~\citep{garg2022can}, while our synthetic setup avoids ambiguity in specifying the compositions.
Many other works that study whether Transformers can compositionally generalize~\citep{csordas2021devil, ontanon2021making}, focus on compositionality within a single forward pass, i.e., the model is not allowed to recursively process its inputs. 
We find the use of intermediate outputs significantly simplifies the problem and, given its popularity in practical scenarios~\citep{kojima2022large, wei2022chain}, our results serve as a demonstration that inference protocols that allow Transformers to recursively refine their outputs can lead to a wide range of capabilities, especially ones that we never explicitly train the model for. 
\vspace*{-4pt}

\section{Formalizing capabilities and compositions}
\label{sec:setup}

As noted by \citet{hupkes2020compositionality}, despite extensive work exploring compositionality in neural networks, the term is often used for several related concepts. To avoid ambiguity, we thus present a definition of a ``compositional model'' that captures our intended notion and, correspondingly, describe the data-generating process used in this work to understand Transformers' ability to compose.
Let $\mathcal{F}$ denote a set of predefined automorphisms, i.e., any given function $F$ from the set defines a map between points from its input space to the same space.
This is motivated by the fact that the input and output domain of a language model are generally the same. 
We define an \textbf{input} $x$ as a combination of two strings $[x_f, x_d]$, where $x_f \in X_{f}^{L}$ is a sequence of $L$ tokens that specify a series of $L$ functions from $\mathcal{F}$, and $x_d \in X_d^K$ denotes a sequence of $K$ tokens to which the series of $L$ functions are applied to. We refer to $x_f$ as \textbf{task tokens} and to $x_d$ as \textbf{data tokens}. For example, let $x_{F_i}$ be the identifier that denotes that function $F_i$ is applied to the data tokens and $x_{d_k}$ denote the $k^{\text{th}}$ token from the vocabulary $X_d$. Assume $L=2$ and $k=1$ and define a sample $x = \left[x_{F_1}, x_{F_2}, x_{d_1}\right]$. Then, a model $M: X_{f}^L \times X_{d}^K \mapsto X_{d}^K$ that takes $x$ as input, is expected to produce the output $F_2 \circ F_1 \left( x_{d_1} \right) $. We use $[L]$ to denote the ordered set $(1, 2, \dots, L)$.

A \textbf{capability}, in our setup, is defined as the ability of a model to accurately represent a function $F \in \mathcal{F}$.
We emphasize that we do not expect pretrained models in practice to perfectly implement an arbitrary function; however, this idealized definition affords us precision and allows us to use accuracy over a random set of inputs to claim a model possesses a certain capability. 
Based on this definition, we intend to understand the set of capabilities---or the set of functions---that a Transformer can implement by composing them. We formalize this as follows.
\begin{definition}[\textbf{Compositionality.}]\label{def:comp}
    We say a model $M(.)$ compositionally generalizes if, for any subset of functions $F_i \in \mathcal{F}$, where $i \in [L]$, $M\left(\left[x_{F_1}, x_{F_2}, \cdots x_{F_L}, x_d\right]\right) = F_{L} \circ \dots \circ F_{2} \circ F_{1}\left(x_d\right)$.
\vspace{-5pt}
\end{definition}

\begin{figure}%[!htb]
    \centering
    \includegraphics[width=0.9\linewidth]{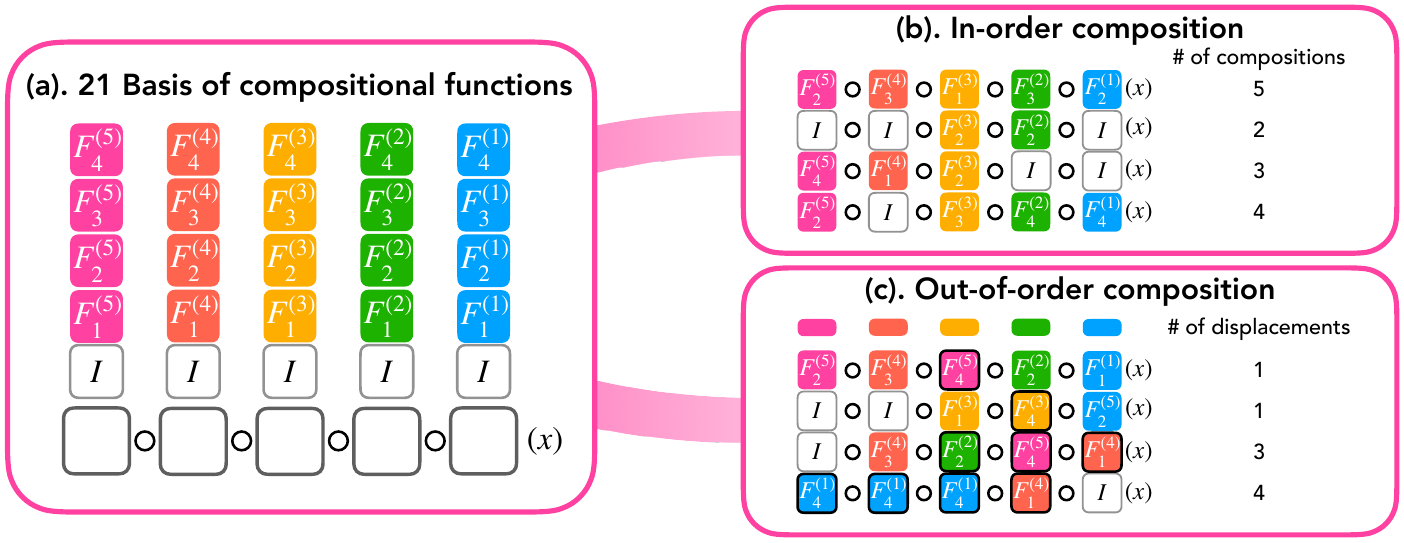}
    \caption{
        \textbf{Data generating process for in-order and out-of-order compositions.}
        (a) Each of the $L=5$ positions is associated with $N=4$ functions $f_i^{[l]}$, in addition to an identity function, resulting in a total of $5 \times 4 + 1 = 21$ basis functions for composition.
        (b) The in-order compositions select functions within the same position while (c) out-of-order compositions allow for selecting functions across positions. Each position also includes the identity function since it allows us to compute compositions of fewer than $5$ functions. In the examples presented in (c), displaced functions are surrounded by a black line, and we then count the number of displaced functions.
    }
    \label{fig:data_gen}
\end{figure}

In practical scenarios, we would not expect the pretraining data to present a capability in all possible scenarios that it can be used in. 
For example, simple arithmetic tasks like multiplication are often only seen in the context of numbers with 1--3 digits in web-crawled data~\citep{razeghi2022impact}, which leads to an inability of the model to perform multiplication in higher order numbers. 
To model this in our setup, we create a spurious correlation between a subset of the functions from $\mathcal{F}$ and the position of their identifiers in the task tokens $x_f$. 
Specifically, we define $\mathcal{F}^{(l)} \subset \mathcal{F}$ as the set of functions that are allowed at the \textbf{position} $l$ in the task tokens $x_f$.
We let $|\mathcal{F}^{(l)}| = N$ for all locations $l$, i.e., $\mathcal{F}$ is partitioned into equally sized subsets and $|\mathcal{F}| = N \times L$. 
The notation $F_{i}^{(l)}$, where $i \in [N]$ and $l \in [L]$, is used to denote the $i^{\text{th}}$ possible function at position $l$.
Based on the above, we define two ways to compose $L$ functions: \textbf{in-order} and \textbf{out-of-order} (see~\Cref{fig:data_gen}).

\begin{definition}[\textbf{In-order vs.\ out-of-order Compositions.}]
Consider the composition $\widetilde{F} = F^{(l_1)} \circ \dots \circ F^{(l_2)} \circ F^{(l_L)}\left(.\right)$, where $l_i \in [L]$. Denote the ordered set $(l_1, l_2, \dots, l_L)$ as $\mathtt{order}(\widetilde{F})$. If $\mathtt{order}(\widetilde{F})$ equals the set $[L]$, we say $\widetilde{F}$ is an \emph{in-order} composition; else, we say it is \emph{out-of-order}. 
\vspace{-5pt}
\end{definition}

Consider a model $M$ that perfectly encodes all $N \times L$ functions from the set $\mathcal{F}$. 
If the model can generalize to \textit{in-order} compositions of these functions, then its set of capabilities will in fact grow to exponentially many functions---$N^L$, to be precise. 
Further, the ability to compose \textit{out-of-order} can increase this set combinatorially, i.e., proportional to $(N \times L)^L$, growing even more quickly compared to the set of in-order compositions. 
Such an ``explosion of capabilities'' would imply that it is difficult to characterize the set of all tasks that a pretrained model can perform, especially since the pretraining data used for training a model is generally unknown and hence it is hard to even characterize what ``atomic'' capabilities the model possesses. 
In our experiments, we find that while Transformers can generalize to both in-order and out-of-order compositions, the pretraining dataset for enabling out-of-order generalization must exhibit some---albeit not huge---diversity (we quantify this further when discussing our results). To empirically characterize out-of-order compositions and discuss the failure modes thereof, we find it useful to define the following notion of \textbf{displacement} (see~\Cref{fig:data_gen}).
\begin{definition}[\textbf{Displacement.}]
    Let $D(s, s')$ denote the hamming distance between two ordered sets $s$ and $s'$. Then, the displacement of a composition $\widetilde{F}$ is defined as $D(\mathtt{order}(\widetilde{F}), [L])$.
\vspace{-5pt}
\end{definition}

\subsection{Experimental Setup and Data-Generating process}\label{ss:gen}
Having defined our notion of compositionality in a pre-trained model, we now briefly discuss the experimental setup used in this work (see~\Cref{s:app:details} for details). 
Specifically, our data-generating process yields inputs consisting of a sequence of 6 \textbf{data tokens}, $x_d \in X_d^6$, where each token is drawn from a vocabulary of size $|X_d|=10$. Each of the 6 elements are drawn uniformly at random, with replacement, from $X_d$. 
We consider \textbf{two families of functions} defined over these data tokens: bijections and permutations (see~\Cref{fig:app:function_set}). 
Specifically, the set $\mathcal{F}_b$ (which we refer to as bijections) consists of all functions that apply a bijection on each of the 6 tokens in an element-wise manner.
The number of such functions is the number of bijections on a single token: there are $10!$ such functions when $|X_d|=10$. 
The second set is $\mathcal{F}_p$, which is the set of all permutations of 6 elements ($|\mathcal{F}_p| = 6!$).
The rationale for selecting these function families is that both $\mathcal{F}_b$ and $\mathcal{F}_p$ are groups with function composition as the group operator. 
Consequently, the composition of two functions is also a group element. 

\begin{figure}[!bt]
  \centering
  \includegraphics[width=0.95\linewidth]{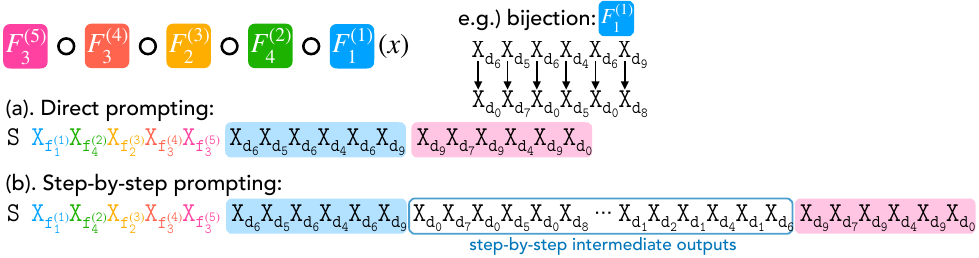}
  \caption{
        \textbf{Direct v.s. Step-by-step prompts.} The task (rainbow) and data (blue) tokens can be completed in two ways. They are followed by: \textbf{(a)} the intermediate outputs of the composition in the 
        step-by-step format or  \textbf{(b)} directly by the final result of compositions in the 
        direct format.
\label{fig:data_prompt}
\vspace*{-10pt}
}
\end{figure}

We consider two formats for representing a sample (see~\Cref{fig:data_prompt}). Both formats start with task tokens $x_f$, that specify the sequence of functions to compose, followed by the data tokens $x_d$.  The \textbf{direct prompt} format follows this with the final output of the function composition, while the \textbf{step-by-step prompt} format follows this with all intermediate outputs of the function composition, similar to chain-of-thought and related protocols~\citep{kojima2022large, nye2021show, wei2022chain}.

We also control the set of \textbf{task tokens} seen during training. In particular, we control compositions in the training data to either only contain in-order compositions, or also include out-of-order compositions.
The training data for $\random$ contains task tokens corresponding to a random subset of the set of all possible in-order compositions. 
%.
The training data for $\base$ contains task tokens where at most one position in
the composition is not the identity function. For example, if we consider $N=4$ and $L=5$
like in~\Cref{fig:data_gen}, then $\base$ contains compositions of functions where at least four of the five
positions are identity, totalling to overall 21 functions.
The set of functions \base~helps us assess whether mere learning of ``atomic''
capabilities is sufficient to yield compositionality in a model. (See~\Cref{s:app:datagen})

We generate 100K samples using the process above for a given prompt format (step-by-step or direct) and with restrictions on the task tokens (in-order, out-of-order, \base , \random ). The model is autoregressively trained on this data using the cross-entropy loss (see~\Cref{s:app:details}). After training, we evaluate whether the model possesses a capability corresponding to a set of composition of functions, by computing the accuracy of the model completion on 1000 different data tokens. The accuracy of a completion is the average accuracy over the last 6 tokens.

\section{Results}
 \label{sec:results}
 
In this section, we systematically investigate the capabilities of an autoregressive Transformer trained
on synthetic tasks with compositional structure. Broadly, we would like to understand how
this structure in the data manifests in the network. We focus on addressing the following questions:
\begin{enumerate}[itemsep=1pt, topsep=-2pt, parsep=2pt, partopsep=2pt]
    \item Do Transformers compostionally generalize to functions not present in the training data and to what extent do they exhibit in-order and out-of-order generalization? % (\Cref{ss:explosion})
    \item How do properties of the training data influence in-order and out-of-order generalization? %(\Cref{ss:ooo})
    \item Are there differences between direct and step-by-step prompt formats? %(\Cref{ss:prompting})
    \item Do Transformers first learn to compose fewer functions before learning to compose many of them? %(\Cref{fig:num_composition})
    \item What is the role of the attention and feed-forward layers? %(\Cref{ss:mechinterp})
    \item Can another popularly used architecture for autoregressive modeling, e.g., LSTMs, compositionally generalize in our setup? %(\Cref{s:app:lstm}) 
\end{enumerate}

We use nanoGPT (\Cref{s:app:details}), a Transformer with 12 layers with each Transformer block identical to the one in~\citet{vaswani2017attention}. We use the same architecture across all our experiments in this section, but provide ablations that vary the number of layers, attention heads, and embedding dimension in \Cref{s:app:transformer_hp}.

\subsection{Combinatorial explosion and Exponential growth in capabilities}
\label{ss:explosion}

Do Transformers only generalize to functions present in the training data or do they reflect compositional structure present in data? In~\Cref{fig:explosion}, we train on data consisting of a small subset of in-order compositions of bijections $\mathcal{F}_b$, in the step-by-step prompt format. 
We consider the composition of 5 functions in both~\Cref{fig:exponential_bijection,fig:combinatorial_bijection}. Each position of the composition can be one of four choices, with the four choices at different positions being different in~\Cref{fig:exponential_bijection} and the same in~\Cref{fig:combinatorial_bijection}. In addition, any position can also be selected to be identity. 

\textbf{We find that Transformers can capture the compositional structure in data and generalize to exponential and combinatorial sets of functions in~\Cref{fig:exponential_bijection,fig:combinatorial_bijection}, despite being trained on an \textit{extremely small} subset of function compositions.} For example, a Transformer trained on 30--100 function compositions, generalizes to 3125 unseen compositions of these functions almost perfectly. 
\begin{figure}[!bt]
    \centering
    \begin{subfigure}[b]{.49\linewidth}
        \centering
        \includegraphics[width=\linewidth]{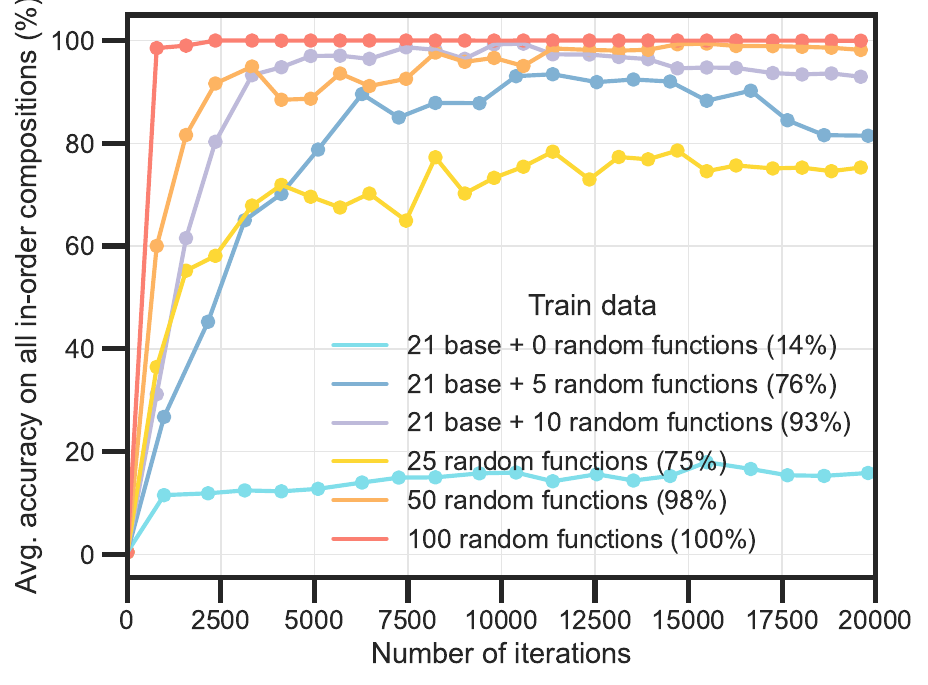}
        \vspace*{-1.5em}\caption{}\label{fig:exponential_bijection}
    \end{subfigure}
    \begin{subfigure}[b]{.49\linewidth}
        \centering
        \includegraphics[width=\linewidth]{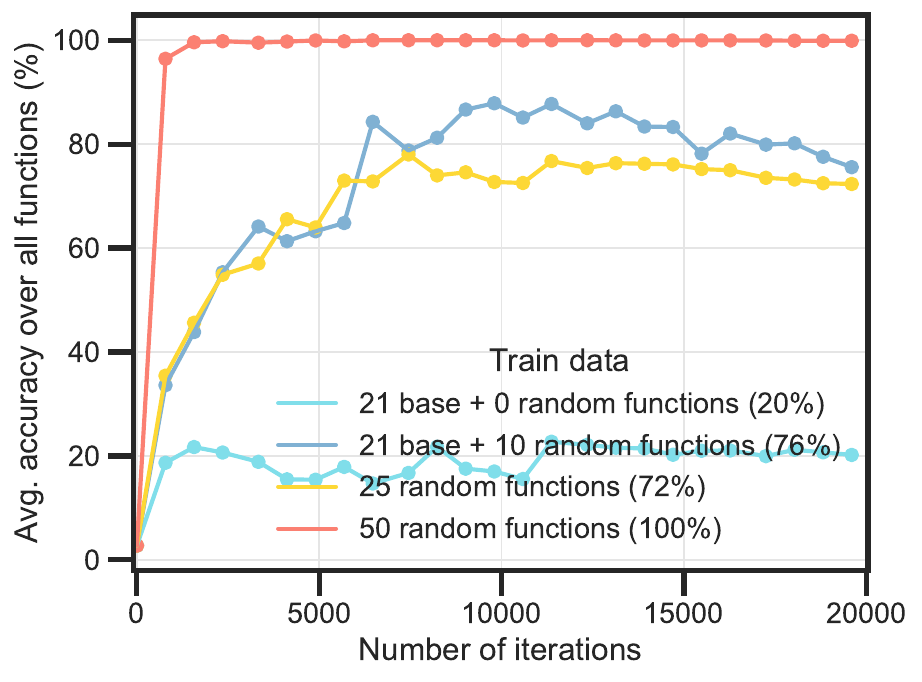}
        \vspace*{-1.5em}\caption{}\label{fig:combinatorial_bijection}
    \end{subfigure}
    \caption{\textbf{Transformers trained on the step-by-step format can generalize to an exponential (a) or combinatorial (b) number of new functions.} 
    We plot the accuracy averaged over all compositions of $L=5$ bijections, where each position of composition has 4+1 choices, with one of them being the identity function. Each curve corresponds to training data generated by a different subset of functions and the model is trained using the step-by-step prompt format. \textbf{(a)} The choice of 5 functions are different at different positions of composition---there are 21 different functions  which can be composed (in-order) in 3125 different ways. \textbf{(b)} The choice of 5 functions are identical across all 5 positions of the composition which means there are 3125 different ways to compose them; only 1365 of them are unique. Both figures are evidence that one can train on a small number of compositions of functions (around 31-100) and generalize to exponentially (a) and combinatorially (b) many functions that would be considered "out-of-distribution".
    \label{fig:explosion}}
\end{figure}
In contrast, we note that LSTMs fail to compositionally generalize in this same setup (\Cref{s:app:lstm}), while Transformers with different numbers of layers and attention heads show compositional generalization (\Cref{s:app:transformer_hp}). This indicates that the \textbf{inductive bias of the architecture contributes to compositional generalization and any autoregressive model is not guaranteed to succeed.} We also observe that \base---which serves as a null model that only trains on the atomic capabilities (or functions)---does not compositionally generalize. Overall, then, we note that compositional generalization occurs with the step-by-step prompt format, provided the right architecture and training data are used.
\begin{figure}[!th]
    \centering
    \includegraphics[width=0.99\linewidth]{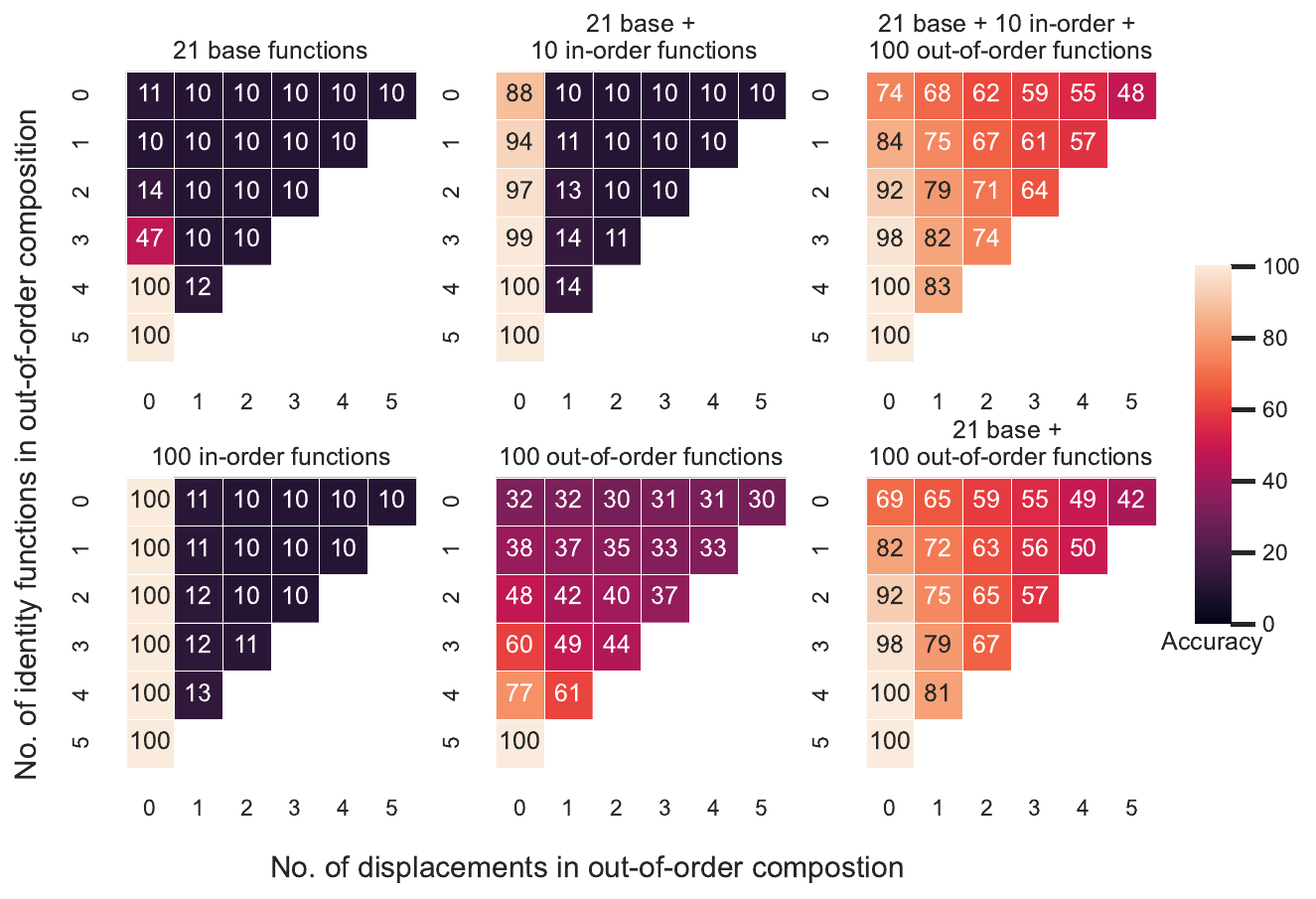}
    \caption{\textbf{The training data determines if a Transformer generalizes to an exponential (in-order generalization) or combinatorial (out-of-order generalization) number of functions.}
    Each sub-plot uses a different subset of functions (from $\mathcal{F}_b$) to generate the training data and we evaluate them on combinatorial set of functions generated from 20+1 functions (one of them being identity). The x-axis varies the number of displacements and the y-axis varies the number of compositions---equivalently the number of functions that are not identity. We make the following observations:
        (1) A Transformer trained on just 31 functions (top-middle) generalize to nearly exponentially many or 3125 compositions of functions.
        (2) All the above configurations do not generalize perfectly to the entire combinatorial set. They however partially generalize to nearly 4 million compositions of functions. The generalization is worse if we increase the number of compositions or displacements (see~\Cref{fig:data_gen} for pictorial description of displacements).
    }
    \label{fig:data_property}
\end{figure}

\subsection{In-order vs. Out-of-order generalization}\label{ss:ooo}

How do biases in the training data influence a Transformer's ability to compose? Are Transformers capable of both in-order and out-of-order generalization or does it depend on the nature of training data? For the functions in~\Cref{fig:exponential_bijection}, the number of in-order compositions is $5^5 = 3125$ and the number of out-of-order compositions is a whopping $(21)^5 = 4084101$; essentially all of these functions are different from the ones seen in the training data. Like in~\Cref{ss:explosion}, we only consider Transformers trained with the step-by-step prompt format on functions from the set of bijections $\mathcal{F}_b$.
In~\Cref{fig:data_property}, we consider the training data to have functions from \base , some in-order and some out-of-order compositions. We fail to see in-order or out-of-order generalization unless the data also includes in-order or out-of-order compositions respectively. \textbf{However, a small number of in-order (10 of them) or out-of-order compositions (100 of them) in the training data results in in-order generalization or limited out-of-order generalization.}
All scenarios in~\Cref{fig:data_property} do not fully generalize to out-of-order compositions. This indicates that out-of-order compositions may require a lot more data compared to in-order compositions.

\subsection{Direct vs. step-by-step compositions}\label{ss:prompting}

Both~\Cref{ss:explosion,ss:ooo} discuss experiments using the step-by-step prompt format, but do these results also hold for direct prompting? 
\Cref{fig:bijectionism} (left) and~\Cref{fig:bijection_direct2} answer this in the negative. 
Specifically, in~\Cref{fig:bijectionism} (left), we consider a setup identical to~\Cref{fig:exponential_bijection} and train on a different number of \random~functions. \textbf{Transformers fail to generalize to new in-order compositions with direct prompting when we consider compositions of bijections from $\mathcal{F}_b$}. We observe this failure even if we train on 2000 of the 3125 possible in-order compositions of functions, i.e., \textit{even if the data has high diversity}. In contrast, in~\Cref{fig:exponential_bijection}, a mere 100 compositions in the step-by-step format suffices to generalize to all possible in-order compositions.

\begin{figure}[!tb]
    \centering
    \begin{subfigure}[b]{.46\linewidth}
        \centering
        \includegraphics[width=\linewidth]{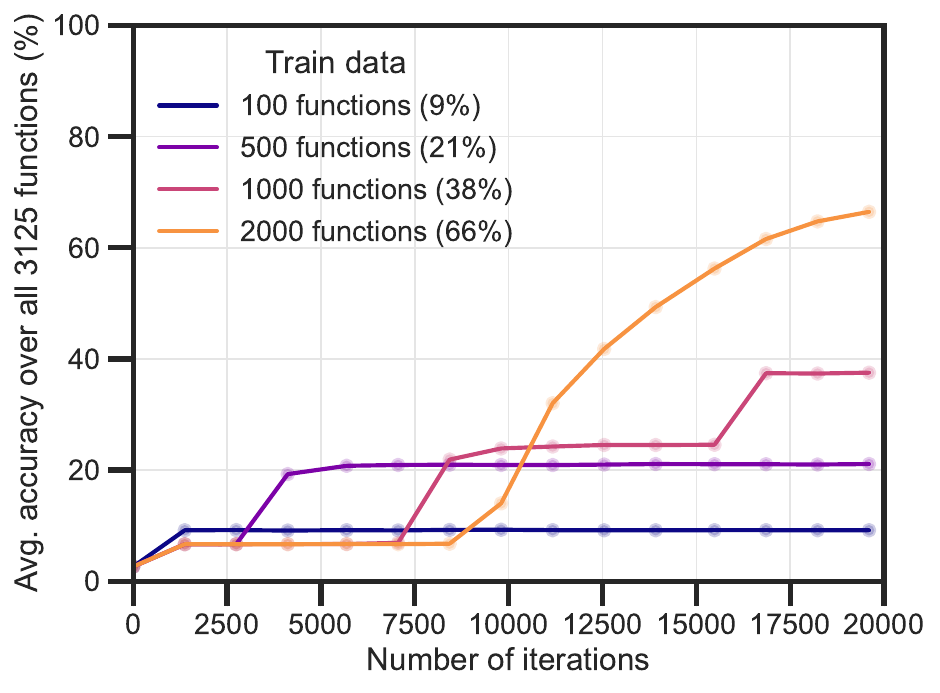}
        \vspace*{-1.0em}%\label{fig:bijection_direct5}
    \end{subfigure}
    \begin{subfigure}[b]{.46\linewidth}
        \centering
        \includegraphics[width=\linewidth]{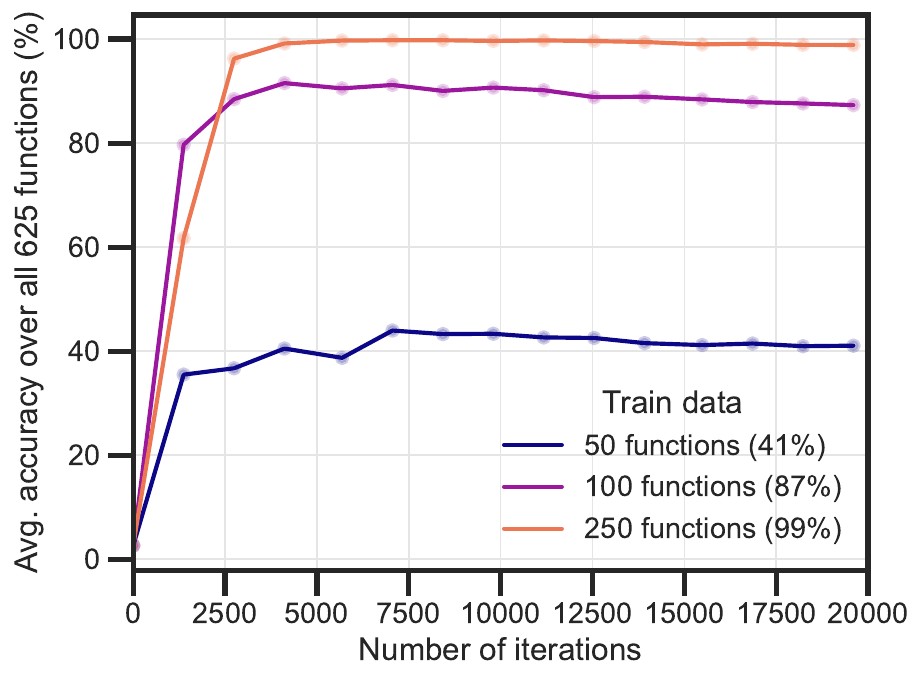}
        \vspace*{-1.0em}%\label{fig:biject9on_permute_map}
    \end{subfigure}
    \caption{\textbf{Compositional generalization is less frequently seen in the direct prompt format.}
        \textbf{(Left.)} We train a Transformer using the direct prompt format on 20+1 bijections with 5 compositions with 4 choices at each position. The model fails to generalize to all 3125 compositions even if it trained on 2000 such functions. 
        \textbf{(Right.)} We train a Transformer using the direct prompt forlat on a composition of two functions, with one function being one of 25 bijections and the other function being one of 25 permutations (totalling to 625 compositions). The model is able to compose previously unseen combinations of functions when trained on 250 of these functions in this scenario.
        \label{fig:bijectionism}
    }
\end{figure}

\textbf{On the other hand, we see in-order generalization if a Transformer is trained on a composition of a a permutation function from $\mathcal{F}_p$ and a bijection function from $\mathcal{F}_b$}. In~\Cref{fig:bijectionism} (right), we train on compositions of two functions, where one position is one of 25 bijections, and the other is one of 25 permutations. We vary the number of compositions seen in the training data and find that 250 compositions in the training data are enough for the model to generalize to all 625 possible compositions of the two functions. We note that bijections and permutations operate on orthogonal features of the input: bijections operate on the value of the token while permutations operate on the position of the token. We speculate that this is important for compositional generalization in the direct prompt format.
% 
% Direct formatted prompts occur less frequently compared to step-by-step compositions and this could be indicative of why chain-of-thought is a popular prompting strategy~\citep{wei2022chain}. A precise answer for when direct prompts can succeed remains unclear though.

\paragraph{Why is compositional generalization harder for direct prompts? (\Cref{s:app:step_vs_direct})}
The ability to run multiple forward passes through the model allows us tackle a richer class of problems~\citep{merrill2023expresssive}. The step-by-step and direct prompt formats differ because the former allows $L$ forward passes through the model, while the latter only allows one forward pass. As a result, we expect for the direct prompt format to enable compositional generalization, it must compute the $L$ steps of the composition in the intermediate layers of the model within a single forward pass itself. 
For example, consider a model that computes the functions $F$ and $G$, and is able to compositionally generalize to function $G \circ F$. Since $G \circ F$ is computed using a single forward pass, $G$ must occur in a layer after $F$ (see also \Cref{fig:decoder_incontext}). However, this model may not generalize to $F \circ G$, since that will require $F$ to occur after $G$ in it model's layers. Hence, to compositionally generalize to both combinations of $F$ and $G$, a model may have to learn copies of $F$ and $G$ at multiple layers. This will likely require training data with large amounts of data diversity so that most combinations of functions are seen by the model during training itself.

We further formalize the intuition above in Appendix~\ref{s:app:theory}. Specifically, in \Cref{s:app:step_vs_direct}, we argue that a model trained with the direct prompt format requires more compositions in the training data, by a factor of $\mathcal{O}(L)$, compared to a model trained with the step-by-step format. In \Cref{thm:direct}, we prove that there exists an $L$-layer Transformer that can compositionally generalize with direct prompting. However, empirically, we find that even with the additional training data, the direct prompt format fails to generalize in~\Cref{fig:bijectionism} (left). This is because the existence of a solution need not guarantee that a Transformer trained with gradient descent converges to that particular minima. The weights can instead converge to a minima that only memorizes compositions present in the training data. 
% Understanding the set of minima and the inductive biases of gradient descent could help us understand why Transformers trained in the step-by-step format compositionally generalize in our setup.

% We believe that direct prompts are unlikely to generalize to the out-of-order compositions or at least require more samples. For example, consider functions $F$ and $G$ and consider a Transformer that computes the function $G \circ F$. Since $G \circ F$ is computed using a single forward pass through a Transformer for direct prompts, $G$ must occur in a layer after $F$ (shown in~\Cref{fig:decoder_incontext}). As a result, the model cannot generalize to $F \circ G$ since $f$ occurs after $G$ in its layers. Hence, a Transformer may have to learn copies of $F$ and $G$ at multiple layers in order to generalize to both $F \circ G$ and $G \circ F$.

\subsection{Towards a mechanistic understanding}
 \label{ss:mechinterp}

In this section, we try to uncover the underlying mechanism for compositional generalization exhibited by Transformers in our setup---particularly for compositions of bijections in the step-by-step prompt format.
Prior work on mechanistic interpretability often studies smaller neural networks to extract insights for larger networks~\citep{nelson2021mathematical,wang2022interpretability,chughtai2302toy}. The rationale relates to the universaility hypothesis~\citep{li2015convergent,olah2020zoom}, which states that networks of different scales are likely to learn similar functions when trained on the same data. In line with this direction, we attempt to understand a 1-layer Transformer\footnote{In fact, we use a deeper model in most experiments in the main paper to elicit maximal performance when using the direct format; the step-by-step format, as we argue in Appendix~\ref{s:app:theory}, can generalize compositionally with fewer layers (one, for in-order generalization).} trained on our data generating process.

To develop a hypothesis for our mechanistic evaluation, we first show in \Cref{s:app:theory_step} the existence of 1-layer Transformers that can compositionally generalize to a simplified version of our task via the step-by-step prompt format. In particular, our construction uses the attention layer to copy the relevant task token---similar to an induction head~\citep{olsson2022context}---and the feed-forward layer to compute an single step of the function composition. The model is run $L$ times serially, where each run computes one step of the function composition. The attention layer uses a position encoding as the key and query to determine which tokens to attend to and propagates the task token as the value.
 
We next evaluate if the theoretical construction, even though a simplification, lines up with empirical evaluations on the actual task. Specifically, we first use linear probing to understand which layers contribute to improvements in the accuracy and then visualize the attention maps to understand which tokens the model attends to.

\textbf{Linear probe accuracy.} In~\Cref{fig:attention} (left), we use a linear probe to analyze the importance of attention layers and MLP layers. Following \citet{geva2022lm}, we fix the parameters of probe to the last linear layer, i.e., the unembedding layer of the trained model. We use a Transformer trained on 100~\random~in-order compositions of 5 functions identical to the model in~\Cref{fig:exponential_bijection}. In~\Cref{fig:app:attention_layers} we show the results of linear probe experiments on Transformers of different sizes. In Transformers of different sizes, we note a sharp increase in accuracy right after an MLP layer, i.e., the accuracy rarely increases after an attention layer.

\textbf{Visualizing attention maps.}
Analyzing the attention maps of a 12-layer Transformer for a discernible pattern can be difficult. We hence analyze the attentin maps of a 1-layer Transformer trained for step-by-step prompts, which surprisingly also exhibits in-order generalization.
In~\Cref{fig:attention} (right), we plot the attention map for a predefined composition of functions from the set $\mathcal{F}_{b}$. Keeping the task tokens to be fixed corresponding to the predefined composition, we sample 1000 data tokens and compute the attention map for the 1-layer model. The average of these maps is reported in the figure. We see that all data tokens attend to:
(i) the task token that specifies the current function to be computed and (ii) the data token that the function is to be applied to.

\textit{The results above remarkably line up with our theoretical construction.} For example, the attention maps in~\Cref{fig:attention} always attend to the relevant task tokens and data token when computing the next step of the composition. The task and data tokens are all embedded in orthogonal spaces, similar to our construction, with the exception of 5 tokens which all correspond to the the identity function (see \Cref{s:app:token_embeddings}). In parallel, the linear probe accuracy for a 1-layer Transformer in~\Cref{fig:app:attention_layers} shows no increase in accuracy after the attention layer (similar to results in Fig.~\ref{fig:attention}), but a sharp increase in accuracy occurs after the MLP layers, indicating that the function is entirely computed in the MLP layers.

\begin{figure}
\centering
\begin{subfigure}[b]{0.49\linewidth}
  \centering
  \centerline{\includegraphics[width=\columnwidth]{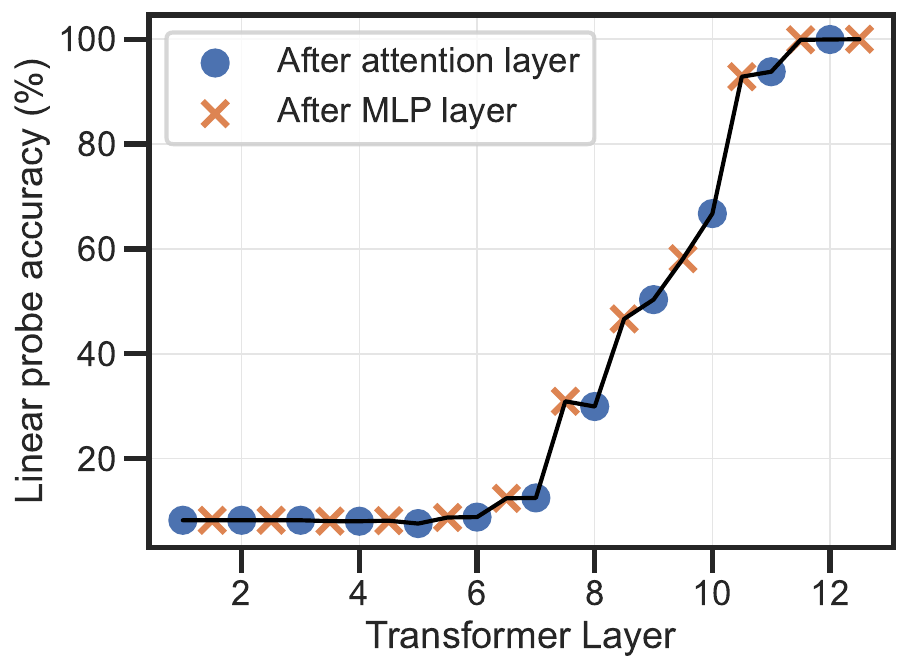}}
\end{subfigure}%
\begin{subfigure}[b]{0.49\linewidth}
  \centering
  \centerline{\includegraphics[width=\columnwidth]{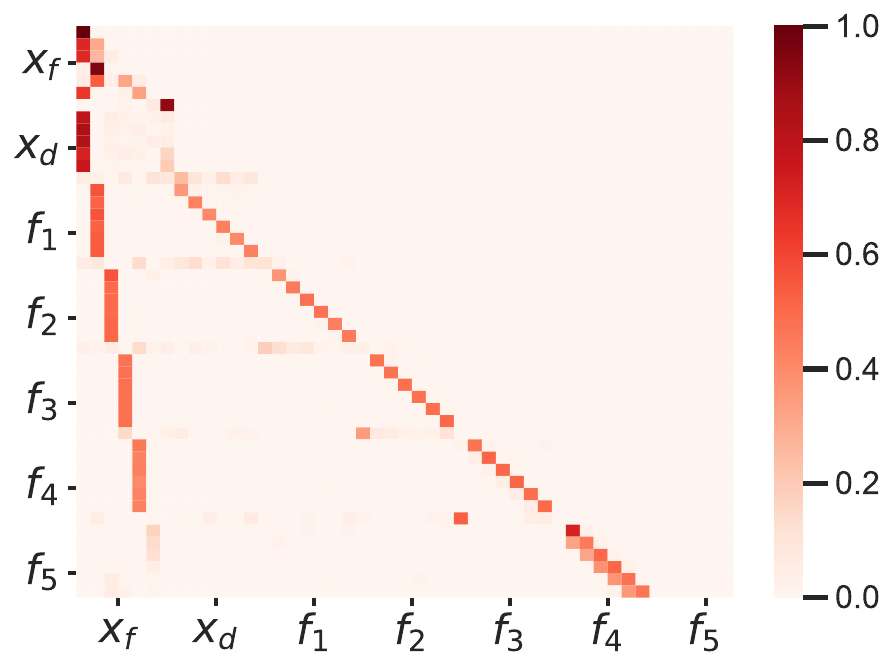}}
\end{subfigure}
\caption{\textbf{(Left.) 
Attention layer picks a function to apply given the current input, and MLP applies the selected function for Transformers trained on compositions of bijections in the step-by-step prompt format.
We see a sharp increases in accuracy after MLP layers in the last few layers of the Transformer.
} We compute the linear probe accuracy---averaged over in-order compositions of functions---after the MLP and attention layers at every layer of the model. 
\textbf{(Right.) Attention is largest at the relevant data and task token.} We plot the causal attention mask of a 1-layer Transformer trained using the step-by-step format on compositions of 5 in-order bijections (setup of~\Cref{fig:explosion}). Keeping the prompt fixed to a specific composition of functions, we plot the attention map averaged over 1000 samples. We observe that the current data token attends to the a specific task relevant to compute the next step of the composition.
}
\label{fig:attention}
\end{figure}

\subsection{Training dynamics}\label{ss:training_dynamics}
\citet{okawa2023compositional} show that different capabilities can emerge multiplicatively over the course of training, i.e., a Transformer first learns functions $F_1$ and $F_2$ before it learns compositions like $F_1 \circ F_2$. In~\Cref{fig:num_composition}, we track the accuracy over the course of training to understand if compositions of fewer functions are learned before compositions of many functions. The setup for this figure is identical to~\Cref{fig:exponential_bijection} with the accuracy faceted by the number of function compositions.
We find that the order in which functions are learned depends entirely on the training data. If the training data consists of \base~and very few in-order compositions, then a Transformer generalizes to fewer compositions (more identities) first before generalizing to compositions of multiple functions. On the other hand, if the model is trained on 25 \random~in-order compositions, then it is better at generalizing to more complex compositions of these functions; this trend is lost when we train on 50 \random~in-order compositions.

\begin{figure}[!tb]
    \centering
        \includegraphics[width=0.99\linewidth]{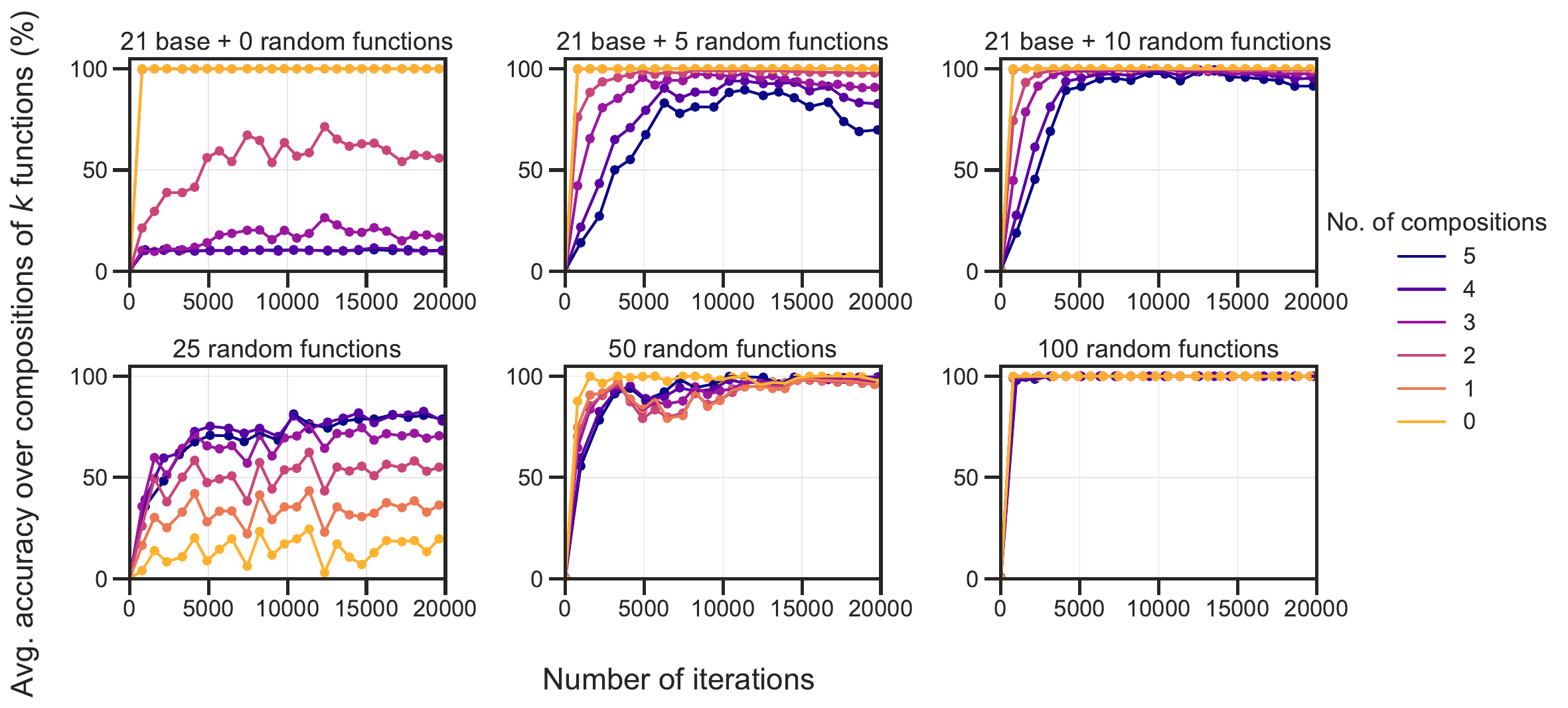}
        \caption{\textbf{
            A Transformer trained on a random subset of functions generalizes first to a composition of more functions before it generalizes to a composition of few of them.}
            Each line is the average accuracy over all composition of $k$ functions and each subplot is a Transformer trained on a different subset of functions. The \base~is trained on the individual functions and these Transformers learn to compose a smaller set of functions (more functions in composition are identity) before learning to compose many of them. The opposite is true when the model is trained on a random subset of 25 compositions of functions.
        } 
        \label{fig:num_composition}
\end{figure}

\section{Conclusion}

Given several recent works focused on prediction or elicitation of capabilities in pretrained models, we ask whether the very motivation guiding these works is tractable: can we possibly characterize all capabilities of a model, specifically a Transformer, pretrained on a compositional data domain? To address this question, we proposed a synthetic, but well-defined, data domain and formalized the notion of a capability as representing a function defined over the domain. Breaking compositional generalization into two relevant scenarios (in-order vs.\ out-of-order), we showed that the compositional structure of the data forces a model to learn to compose at relatively minimal data diversity, which indicatively address our primary question: an appropriate prompt could make the model compose its capabilities, yielding an ``explosion of capabilities''. This can arguably make tractable analysis of capabilities in a pretrained model relatively difficult.

\section*{Acknowledgements}

RR thanks Kento Nishi, Gautam Reddy and Eric Bigelow for their discussions at the early
stages of this project. RR thanks AWS AI, for their gift to Penn Engineering's ASSET Center for Trustworthy AI. ESL was
partially supported by the National Science Foundation (IIS-2008151).

\section*{Author Contributions}

ESL and RR conceived the initial project direction and defined the problem setup with with
inputs from HT and MK. The experiments were led by RR with inputs from ESL, HT and MK. The
writing of the introduction and related work was led by ESL with help from HT and RR. 
RR, ESL and HT extensively collaborated on the methods section. The results and appendix
were led by RR. The expository figures were created by HT and RR. HT and RPD acted as
advisors in the work.

\bibliography{arxiv}
\bibliographystyle{icml2024}

\newpage
\appendix
\onecolumn

\newpage
\section{Experimental Details}\label{s:app:details}

\subsection{Training methodology}

\paragraph{Transformer architecture}

\begin{wrapfigure}{r}{0.5\textwidth} 
    \vspace*{-12pt}
    \centering
    \includegraphics[width=0.99\linewidth]{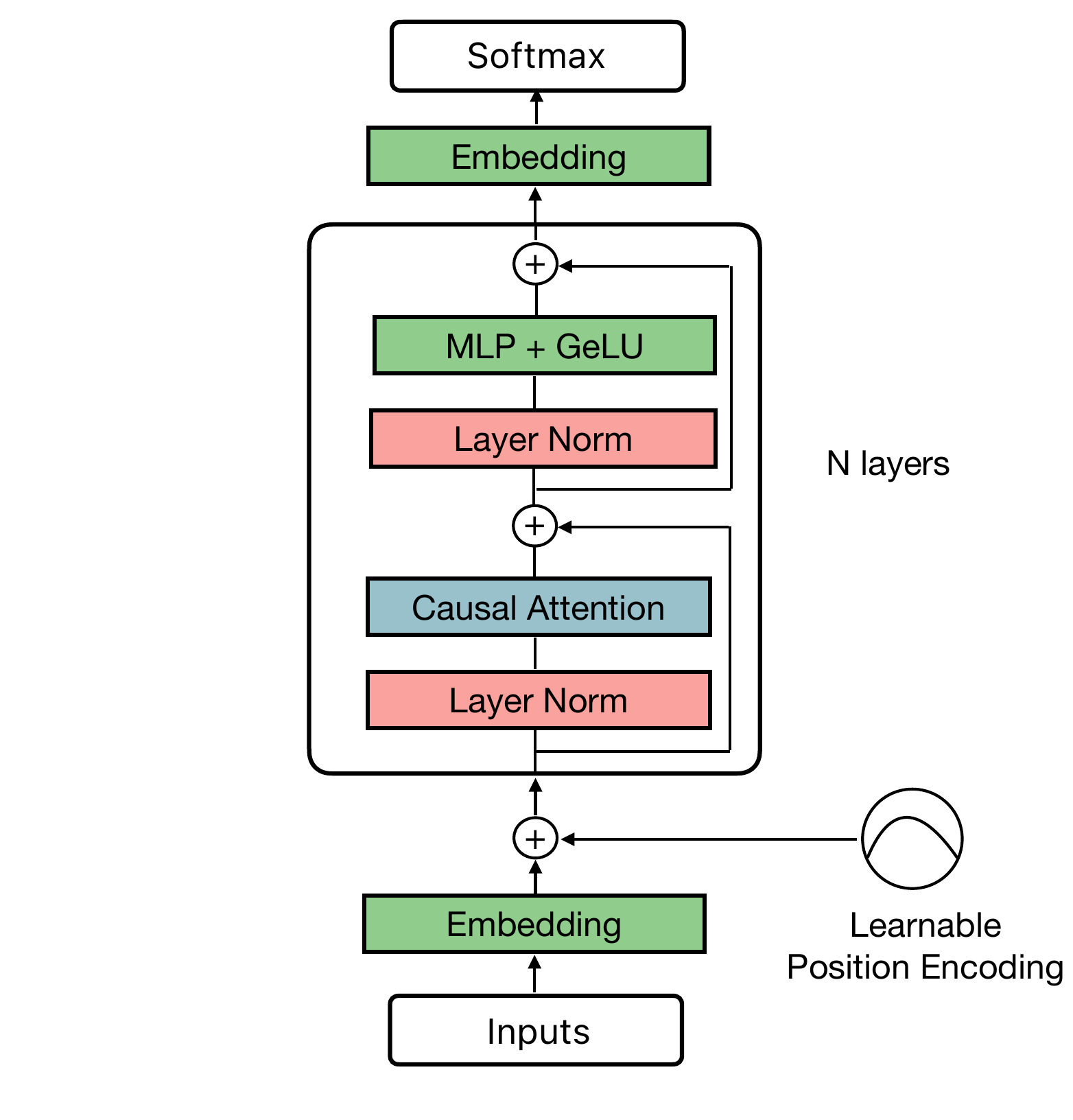}
    \caption{
    We use nanoGPT as the Transformer architecture in all our experiments. The core Transformer block is a LayerNorm, a causal attention block, followed by another layer-norm and a 2-layer multi-layer perceptron (MLP). The Transformer block has two residual connections.}
    \label{fig:app:transformer}
    \vspace*{-20pt}
\end{wrapfigure}

We use nanoGPT\footnote{https://github.com/karpathy/nanoGPT} with 12 layers, 12 attention heads and an embedding dimension of size 120. 
Each transformer block contains a causal attention layer, layer-norms, residual connections and an MLP (see~\Cref{fig:app:transformer}). The MLP contains two fully-connected layers sandwiched by a GELU layer~\citep{hendrycks2016gaussian} The first fully-connected layers has a hidden layer with size 4 times the embedding dimension (480) and the second hidden layer has a size equal to the embedding dimension (120).

The input tokens are converted to one-hot vectors before being passed through to the model. We do not use dropout or biases in the LayerNorm layers. We use weight-tying~\citep{press2016using}, i.e., the input and the output embedding layers share weights.  Finally, we make use of mixed-precision (bf16 in torch) to speed-up training.

\paragraph{Loss and Optimizer}

Models are trained using an autoregressive objective to predict the next token using the cross-entropy loss. Specifically, assume a sequence of tokens of $t$ tokens denoted by $x_{1:t}$.  Let $p_{w}(y \mid x_{1:t}$) denote the probability distribution over the next token as predicted by a model with weights $w$. For a sequence $x_{1:T}$ of length $T$, the autoregressive objective is
\[ L(w) = - \sum_{t=1}^{T-1} \log p_{w}\left(y=x_{t+1} \mid x_{1:t} \right).\]

Training is performed for 100 epochs with a cosine-annealed scheduled with warmup. We use an initial learning rate of 3e-4 annealed eventually to 6e-5. We use AdamW as the optimizer ($\beta_1=0.9$ and $\beta_2 = 0.95$) with a weight decay of $10^{-3}$ and a batch-size of 512. We also make use of gradient clipping with a magnitude of 1.

\subsection{Data generating process}\label{s:app:datagen}

\paragraph{Data and task tokens.} Both data and task tokens are converted to one-hot vectors before being fed to the Transformer. The set of data tokens is denoted by $X_d$ and the size of the vocabulary, $|X_d|$, is 10 in all our experiments. The data tokens in the input $x_d \in X_d^6$ is a sequence of $6$ tokens and is the input to the function composition. The 6 tokens are sampled uniformly at random from $X_d$ with replacement.

There are two sets of functions considered in this work. The set of functions $\mathcal{F}_b$ (which we refer to as bijections) applies a lookup table in an element-wise fashion to each of the 6 tokens in $x_d$. The set of functions in $\mathcal{F}_p$ permute the 6 tokens in $x_d$. The family of functions in $\mathcal{F}_b$ and $\mathcal{F}_p$ are described in~\Cref{fig:app:function_set}. Each function from  $\mathcal{F}_p$ and $\mathcal{X}_b$ has its own task token in $X_F$.

The input starts with a sequence of $L$ task tokens $x_f \in X_F^L$. The number of compositions is generally $L=5$, but in a few experiments like Figs.~\ref{fig:bijection_direct2},~\ref{fig:bijectionism}~(Right), $L=2$. 

\paragraph{Sampling task tokens}

The task tokens can be sampled such that they satisfy certain properties. For example, let us consider the composition of two functions---one from the set $\mathcal{F}_{1} \subset \mathcal{F}_p$ and another from $\mathcal{F}_{2} \subset \mathcal{F}_b$ (which is the setting in Fig.~\ref{fig:bijectionism}~(Right)). We can restrict the training data to compositions from the set $\mathcal{F}_{2} \circ \mathcal{F}_{1}$ which are in-order compositions (see Fig.~\ref{fig:data_gen}). Alternatively, we can also choose to include out-of-order compositions, which include compositions from $\mathcal{F}_{1} \circ \mathcal{F}_{1}, \mathcal{F}_{2} \circ \mathcal{F}_{2}$ and $\mathcal{F}_{1} \circ \mathcal{F}_{2}$. In Fig.~\ref{fig:bijectionism}~(Right), we restrict our training and evaluation to in-order compositions of functions and we observe that training on a subset of the elements from $\mathcal{F}_2 \circ \mathcal{F}_1$ suffices to compositionally generalize all functions in the set.

Two other commonly used subsets of functions are \base~and \random. Consider $\mathcal{F}_1, \mathcal{F}_2, \ldots, \mathcal{F}_5 \subset \mathcal{F}_b$. The set \random~considers $k$ functions from the set $\mathcal{F}_5 \circ \mathcal{F}_4 \circ \cdots \circ \mathcal{F}_1$ which are drawn uniformly at random. 

\base~is used to test if the compositionality is seen when the Transformer is trained on the individual functions from $\mathcal{F}_i$ for all $i \in [5]$. In the training data, all compositions have 4 of the 5 functions to be the identity function $I$, i.e it considers compositions of the form $I \circ I \circ \mathcal{F}_3 \circ I \circ I$ or $I \circ \mathcal{F}_4 \circ \cdots \circ I$. There are a total of $1 + \sum_{i=1}^5 \mathcal{F}_i$ such functions; the 1 is when all 5 functions in the composition are identity. The model is never trained on the composition of two or more functions, and at least compositions of 3 functions are necessary to generalize to all in-order compositions~\Cref{fig:max_compose}.
\begin{wrapfigure}{r}{0.6\textwidth} 
    \centering
    \includegraphics[width=0.65\linewidth]{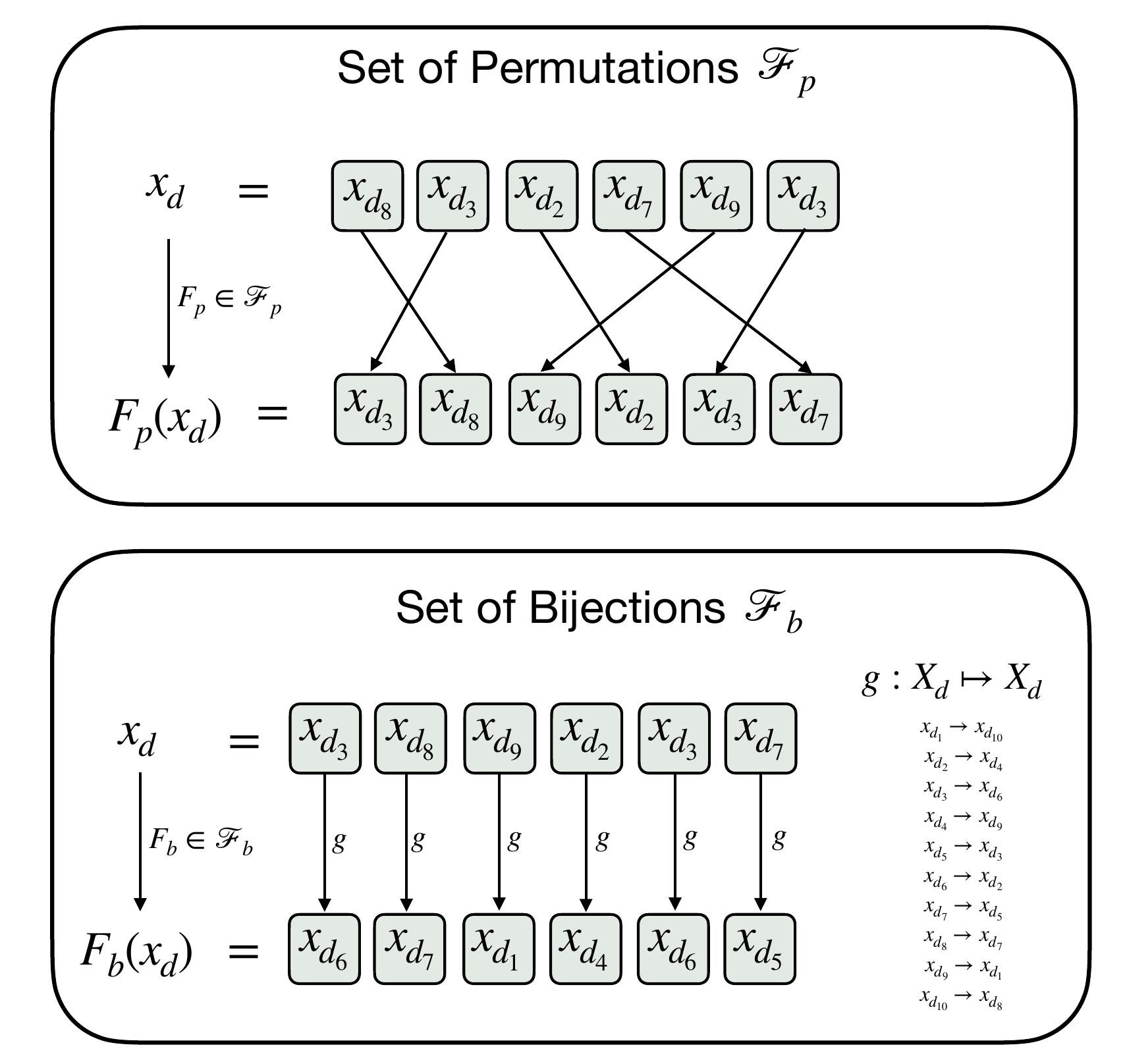}
    \caption{A permutation from $\mathcal{F}_p$ permutes the 6 tokens in the input $x_d$. A bijection from $\mathcal{F}_b$ applies a lookup table to each of the 6 tokens individually.
    }
    \label{fig:app:function_set}
    \vspace*{-30pt}
\end{wrapfigure}

\paragraph{Generating a sequence of tokens}
A sequence starts with a sequence of two task tokens $x_f = [x_{F_1}, x_{F_2}]$ followed by a sequence of data tokens $x_d$. The sequence can either be presented in: (i) The step-by-step format,
where the intermediate outputs are also included in the sequence; e.g., the sequence in the step-by-step format would look like $[x_{F_1}, x_{F_2}, x_d, F_1(x_d), F_2(F_1(x_d))]$ (see \Cref{fig:decoder_composition}) or (ii) The direct format, which does not include the intermediate outputs of the composition in the sequence and an example of such a sequence is $[x_{F_1}, x_{F_2}, x_d, F_2(F_1(x_g))]$ (see \Cref{fig:decoder_incontext}).

The step-by-step and direct formats are also discussed in~\Cref{fig:data_prompt}. The training data consists of 100,000 sequences for all experiments in one of the two formats.

\begin{figure}[!htb]
    \centering
    \begin{subfigure}[b]{.675\textwidth}
        \centering
        \includegraphics[width=0.95\textwidth]{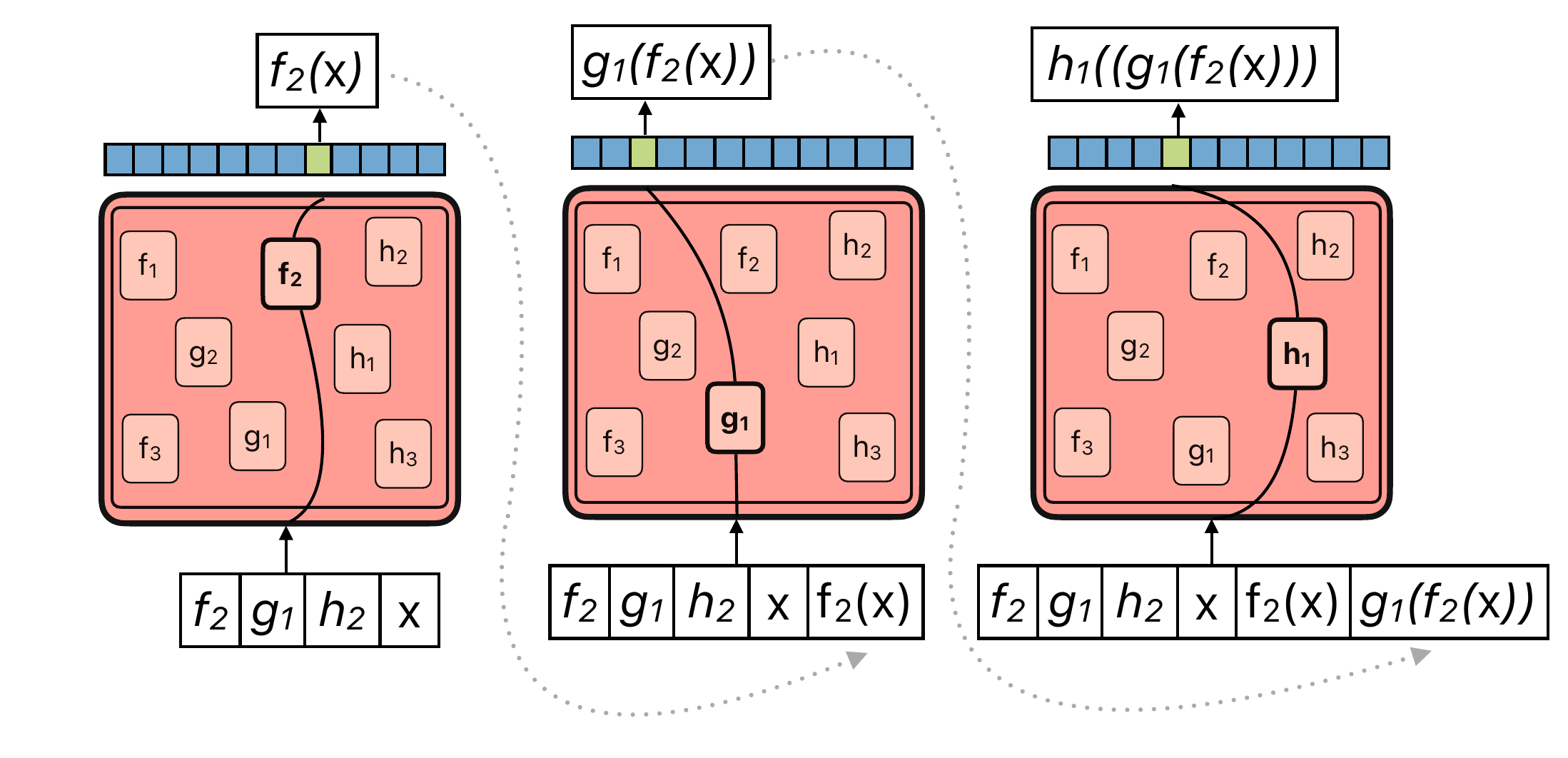}
        \vspace*{-1.0em}
        \caption{}
        \label{fig:decoder_composition}
    \end{subfigure}
    \begin{subfigure}[b]{.315\textwidth}
        \centering
        \raisebox{2.0mm}{
        \includegraphics[width=0.57\textwidth]{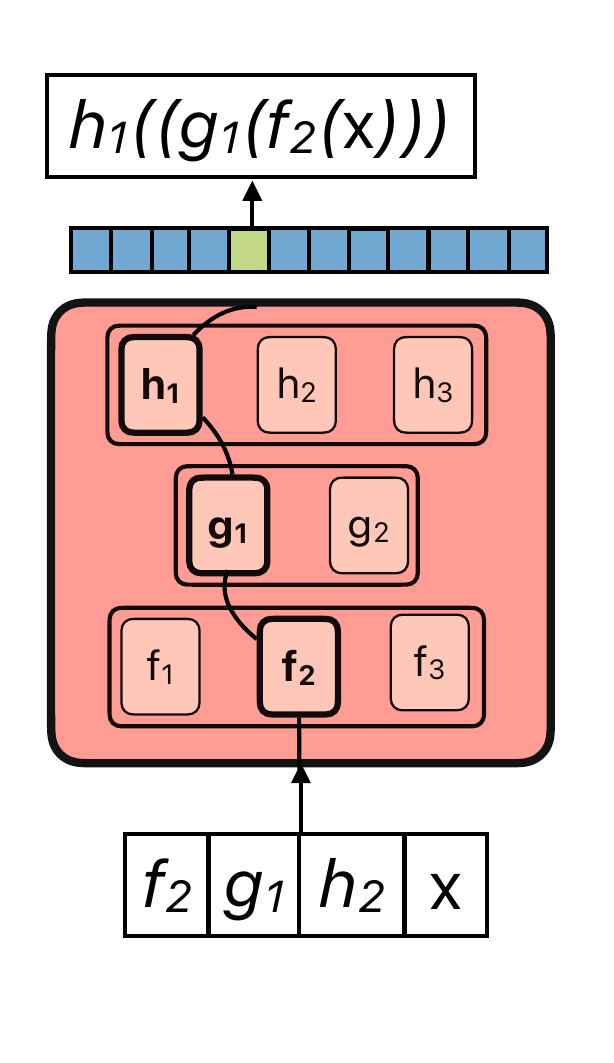}}
        \vspace*{-1.0em}
        \caption{}
        \label{fig:decoder_incontext}
    \end{subfigure}
    \label{fig:deocder}
    \vspace*{-1.5em}
    \caption{
        \textbf{Step-by-step composition v.s. Direct composition.} We test two possible routes for compositions.
        \textbf{(a)} Step-by-step prompting, which allows for generating intermediate outputs. 
        \textbf{(b)} Direct prompting, where the model must compose the functions without
the intermediate outputs.}
\end{figure}

\paragraph{Evaluating compositions}
When evaluating trained models, we evaluate on 1000 different inputs for every composition of functions. Since Fig.~\ref{fig:data_property} requires us to evaluate on a combinatorial set of functions, we sample 1000 functions (or the total number of functions, whichever was lower) for each cell which can be identified by the displacement and number of compositions; we then compute the accuracy averaged over those functions to populate the cell.
The accuracy of a completion is calculated by averaging the accuracy of the last six tokens. We see that qualitative trends do not change when we use different metrics~\Cref{fig:app:metrics}.

\paragraph{Computing linear probe accuracy}  We consider the outputs after every attention block and every MLP block (including the residual stream in both cases). We then pass these outputs through the final embedding layer and a Softmax layer to get predictions over the next token. We use these predictions to compute the accuracy at that layer. The accuracy is averaged over 1000 different input data tokens and for 200 different compositions of functions.

\section{Additional Experiments}\label{s:app:results}

\subsection{Sweeping hyper-parameters of the Transformer}\label{s:app:transformer_hp}

\begin{wrapfigure}{r}{0.5\textwidth} 
    \vspace*{-18pt}
    \centering
    \includegraphics[width=0.7\linewidth]{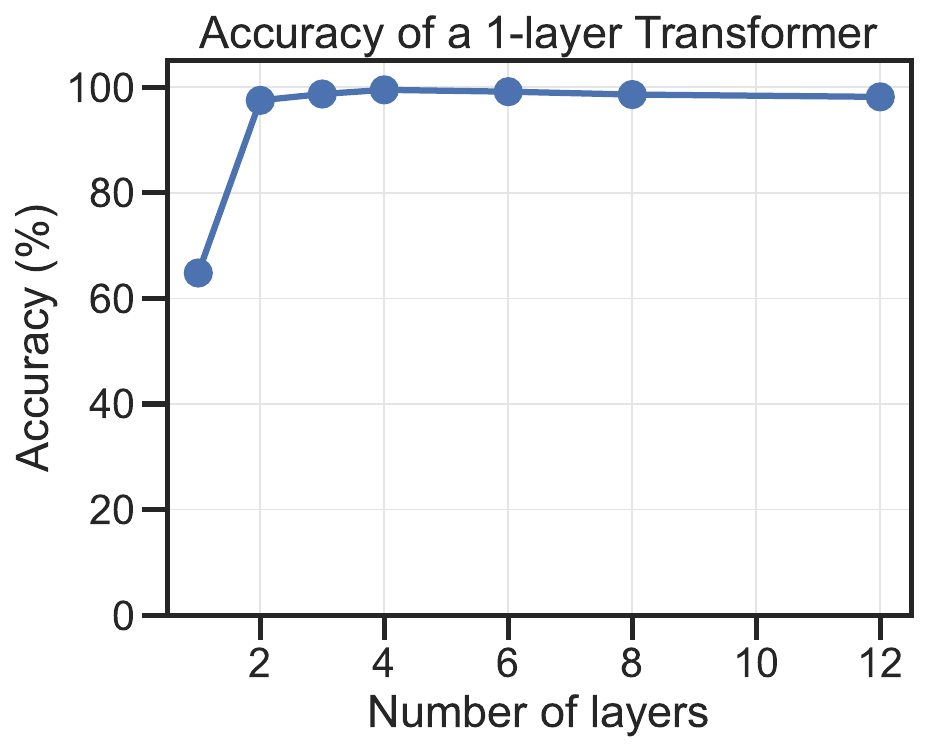}
    \caption{
    \textbf{Transformers requires at least 2-3 layers for compositional generalization with the direct prompt format.}
    We vary the number of layers in the Transformer and train on direct composition in a setup identical to~\Cref{fig:bijectionism}~(Right).
    \label{fig:app:arch_direct}
    }
    \vspace*{-10pt}
\end{wrapfigure}

We vary the number of layers, the number of attention heads, and the embedding dimension of the nanoGPT model in Fig.~\ref{fig:app:arch}. We consider a setup identical to Fig.~\ref{fig:explosion}; all models are trained on 50 \random~in-order compositions of 5 bijections. We report accuracy averaged over all 3125 in-order compositions.

We make the following observations.
\begin{enumerate*}
    \item Most surprisingly, the accuracy reduces as the number of layers become \textit{huge} for this compositional task; we expect that this is due to issues with optimization of a large depth model.
    \item The accuracy does not change with the number of attention heads for a 1-layer Transformer.
    \item The accuracy increases as we increase the embedding dimension and the model under fits the training data when the embedding dimension is too small.
\end{enumerate*}

\begin{figure}[!htb]
    \centering
    \includegraphics[width=0.32\textwidth]{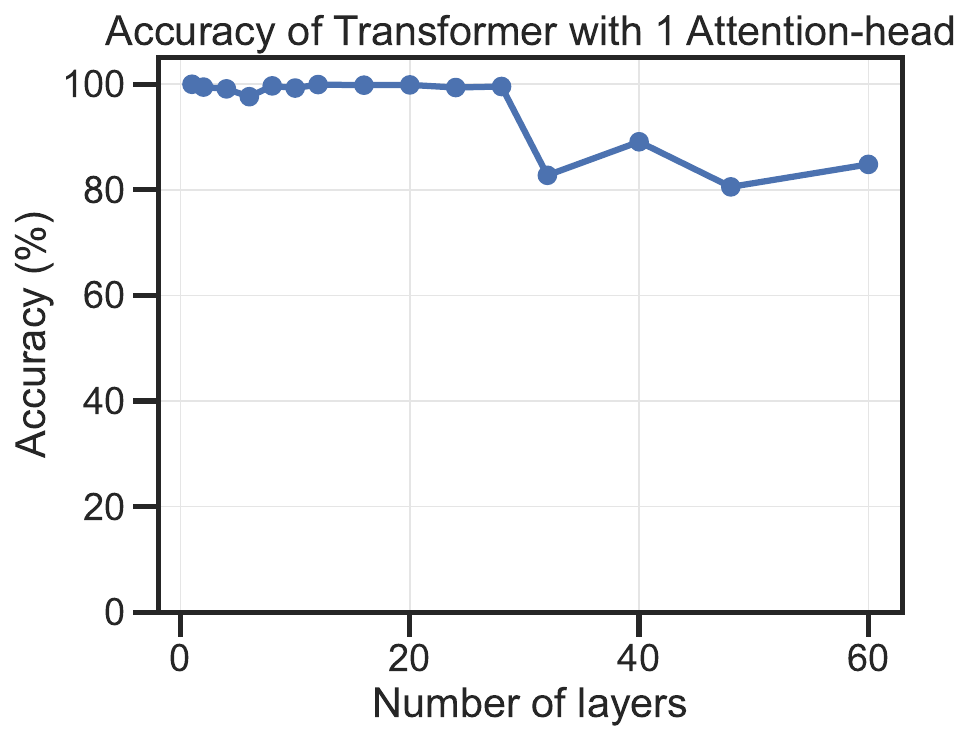}
    \includegraphics[width=0.32\textwidth]{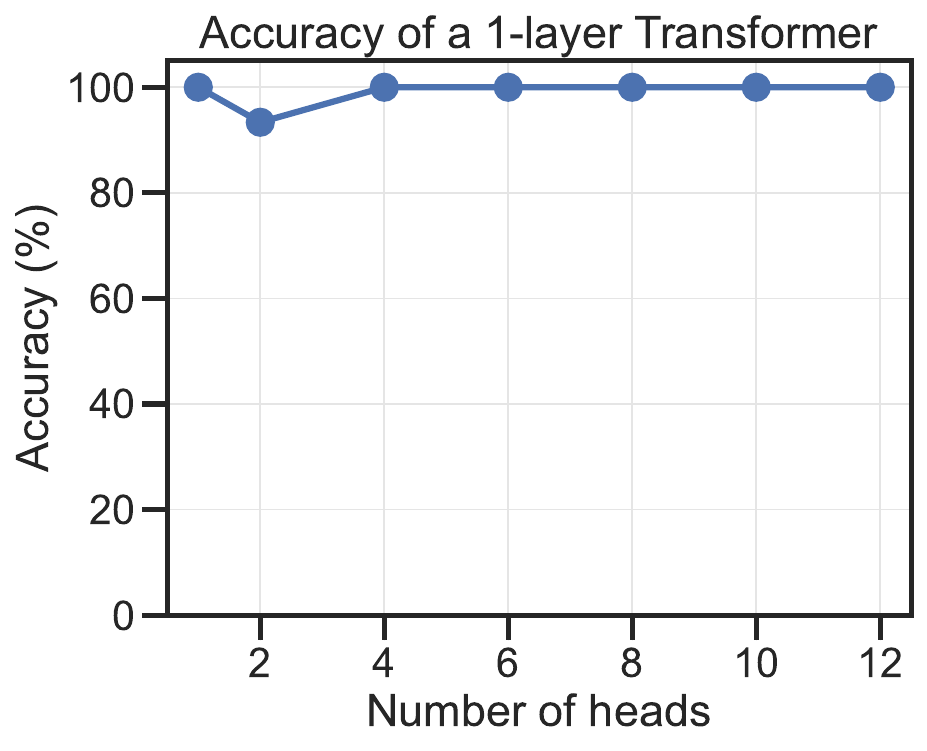}
    \caption{\textbf{We see compositionality in Transformers even if we change the number of layers and attention heads.} Compositionality is seen even in a 1-layer Transformer when trained with the step-by-step prompt format on 50 in-order compositions of bijections. However the ability to compose degrades as we increase the number of layers in the Transformer.}
    \label{fig:app:arch}
\end{figure}

\subsection{LSTMs do not learn to compose}\label{s:app:lstm}
\WFclear
We report results on autoregressively trained LSTMs using the direct prompt format from \Cref{tab:app:direct_lstm} and the step-by-step prompt format in \Cref{tab:app:lstm}. \textbf{LSTMs fail to generalize outside of the training data while Transformers generalize compositionally in both these scenarios}. \textit{This points to an inductive bias that helps Transformers trained with an autoregressive objective generalize.} Specifically, our mechanistic evaluation in Sec.~\ref{ss:mechinterp} shows this is likely attributable to the use of Attention.

The LSTMs are trained using the same data using the autoregressive objective defined in~\Cref{s:app:details}. We use the AdamW optimizer with learning rate equal to 3e-4 ($\beta_1=0.9$ and $\beta_2=0.95$), batch size of 512 and weight decay of 1e-4 for 150 epochs. As is common, we do not use a positional embedding, since the architecture is not permutation invariant. 

\begin{wraptable}{r}{0.65\textwidth} 
    \vspace*{-15pt}
    \centering
    \begin{tabular}{ c | c c}
      \toprule
       \rowcolor{gray!10} % Header row color :w
          & \multicolumn{2}{c}{Hidden dimension} \\
       \rowcolor{gray!10} % Header row color :w
      Layers & 256 & 5124 \\
      \midrule
      \midrule
      1 & 22.5 & 46.0 \\
      2 & 33.4 & 69.1 \\
      \bottomrule
    \end{tabular}
    \caption{
        \textbf{LSTMs fail to compose in the direct prompt format.}
        We train an LSTM on 250 composition of two functions (one permutation and one bijection) in the direct prompt format and tabulate the accuracy (\%); the setup is identical to~\Cref{fig:bijectionism} (Right). 
    }
    \label{tab:app:direct_lstm}
    %\vspace*{-20pt}
\end{wraptable}

The inputs are passed through an input embedding layer before being passed to the LSTM and the outputs of the LSTM are also passed through a linear layer which outputs the logits. In our experiments, we vary the number of stacked LSTMs (or no. of layers) and the dimension of the internal hidden vector.

Despite our attempt to train multiple different LSTMs with the best set of hyper-parameters, we observe that they do not show any compositional generalization on all our synthetic setups. This observation is further evidence for our hypothesis that the attention layers are important for compositionality.

\begin{table}[!htb]
    \centering
    \begin{tabular}{ c | c c c c }
      \toprule
       \rowcolor{gray!10} % Header row color :w
          & \multicolumn{4}{c}{Hidden layer dimension} \\
       \rowcolor{gray!10} % Header row color :w
       Layers & 120 & 256 & 512 & 1024 \\
      \midrule
      \midrule
      1 & 16.2   & 36.2 & 99.9 &  99.9 \\
      2 & 60.3   & 99.3 & 99.9 & 99.8 \\
      4 & 18.7   & 100.0 & 100.0 & 9.9 \\
      \bottomrule
    \end{tabular}
    \hspace{1.5em}
    \begin{tabular}{ c | c c c c }
      \toprule
       \rowcolor{gray!10} % Header row color :w
          & \multicolumn{4}{c}{Hidden layer dimension} \\
       \rowcolor{gray!10} % Header row color :w
       Layers & 120 & 256 & 512 & 1024 \\
      \midrule
      \midrule
      1 & 9.3  & 10.3 & 20.1 &  22.9 \\
      2 & 12.4 & 21.3 & 25.3 & 28.8 \\
      4 & 6.6  & 13.9 & 17.6 &  10.0 \\
      \bottomrule
    \end{tabular}
    \caption{\textbf{LSTMs fail to compose in the step-by-step prompt format.}
    We train autoregressive LSTMs on 50 in-order compositions of 5 bijections from $\mathcal{F}_b$ in the step-by-step format and tabulate the accuracy (\%); The setup is identical to~\Cref{fig:explosion}. We evaluate the LSTM on the \textbf{(left)} compositions seen during training and \textbf{(right)} in-order compositions not seen during training. LSTMs fail to generalize to functions outside of the training data while transformers generalize compositionally in the same setting. }
    \label{tab:app:lstm}
    
\end{table}

\subsection{Attention Masks}

\textbf{Detailed setup.} We train a 1-layer Transformer on a composition of 50 \random~in-order compositions of 5 bijections in the step-by-step prompt format. We visualize the attention masks for a fixed sequence of task tokens, averaged over 1000 different data tokens in Fig.~\ref{fig:attention}(right). We found the attention masks to be identical across different choices of the task tokens. Each row corresponds to a causal attention mask for a single token and sums up to 1. At any given row, the attention is over two elements---the task token and the intermediate output of the composition. The five contiguous blocks along the columns correspond to the five steps of composition. These preliminary results indicate that it is possible to build a complete mechanistic understanding of attention for compositional tasks (see also Sec.~\ref{s:app:theory}).

\subsection{Probing the layers in Transformers of different sizes}

\begin{figure}[!htb]
    \centering
    \includegraphics[width=0.85\textwidth]{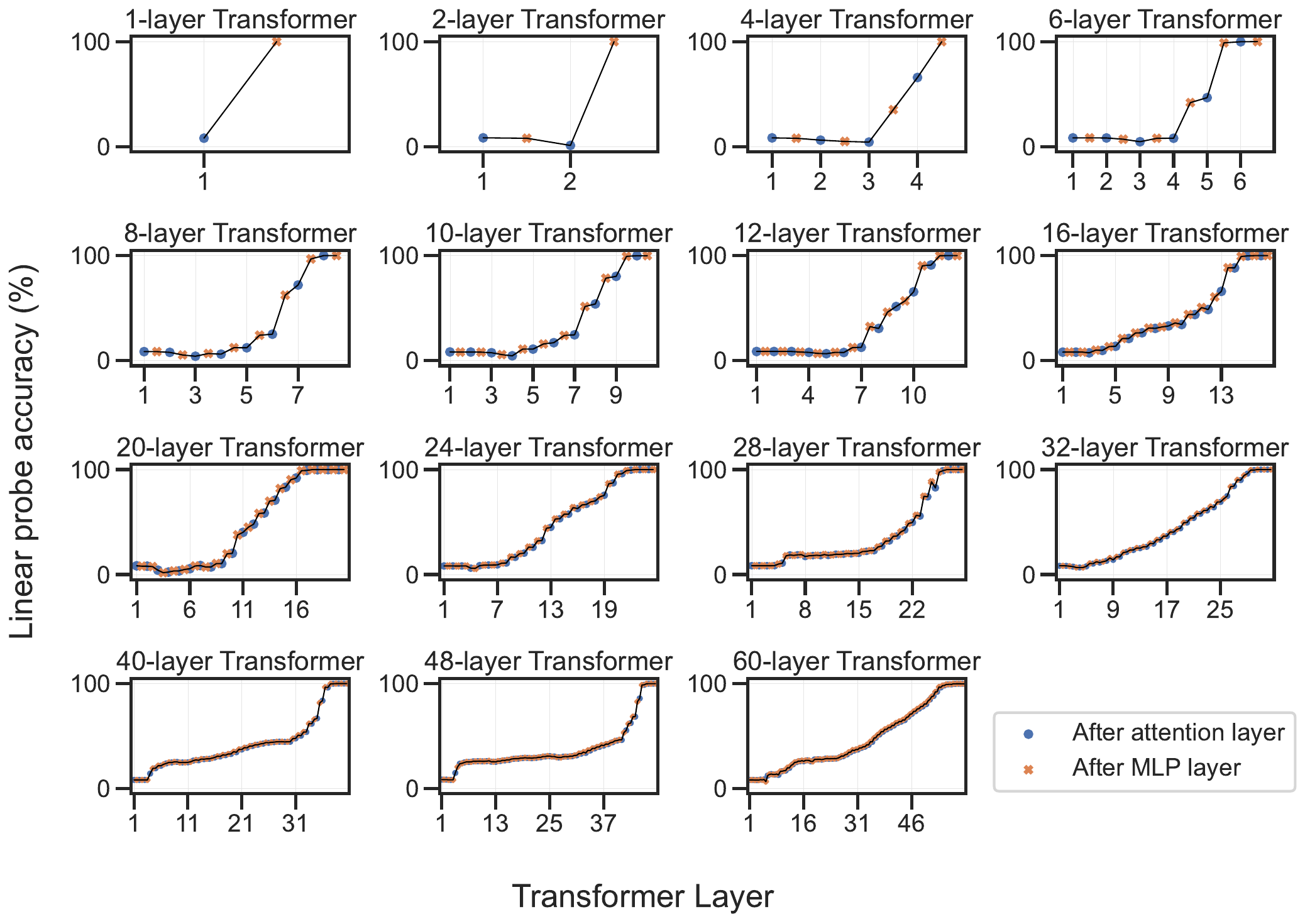}
    \caption{\textbf{We use a linear probe to study the accuracy at different layers on Transformers of different sizes.}
    Most architectures see an increasing in accuracy in the latter half of the Transformer. The increase in accuracy is more gradual for Transformers with more layers. The accuracy increases sharply after an attention layer across all architectures. 
    }
    \label{fig:app:attention_layers}
\end{figure}
In this section, we consider an experimental setup that is identical to the linear probe experiments in Fig.~\ref{fig:attention}. We compute the probe accuracies for Transformers with different number of layers in Fig.~\ref{fig:app:attention_layers}. Across all models, we observe that accuracy increases in the last few layers. Furthermore, we also observe a sharp increase in accuracy right after the MLPs in the last few layers of the transformer. 

We saw in Fig.~\ref{fig:attention}(right) that the attention masks for a 1-layer model seem to select an input and a task token to operate on at every step of the composition. We hence believe that attention has a huge role in compositionality and propose the following hypothesis:  
% \begin{enumerate*}
    % \item LSTMs fail to compose functions not present in the training data. We hypothesize that a lack of attention contributes to this failure.
    % \item 
The probe accuracy after some MLPs see a sharp in increase in accuracy because the attention layers play a critical role in selecting the right inputs to pass to the MLP. Specifically, unlike the 1-layer model, we suspect functions are now distributed across the model layers instead of being localized in the first MLP layer. Consequently, similar to the 1-layer model, attention heads at different layers will infer if the relevant functions implemented in MLP layers in that block are part of the prompt; if so, they transfer the input data through said function.
% \end{enumerate*}

\subsection{Another failure with the direct format with bijections}

\begin{figure}[!htb]
    \centering
    \includegraphics[width=0.70\textwidth]{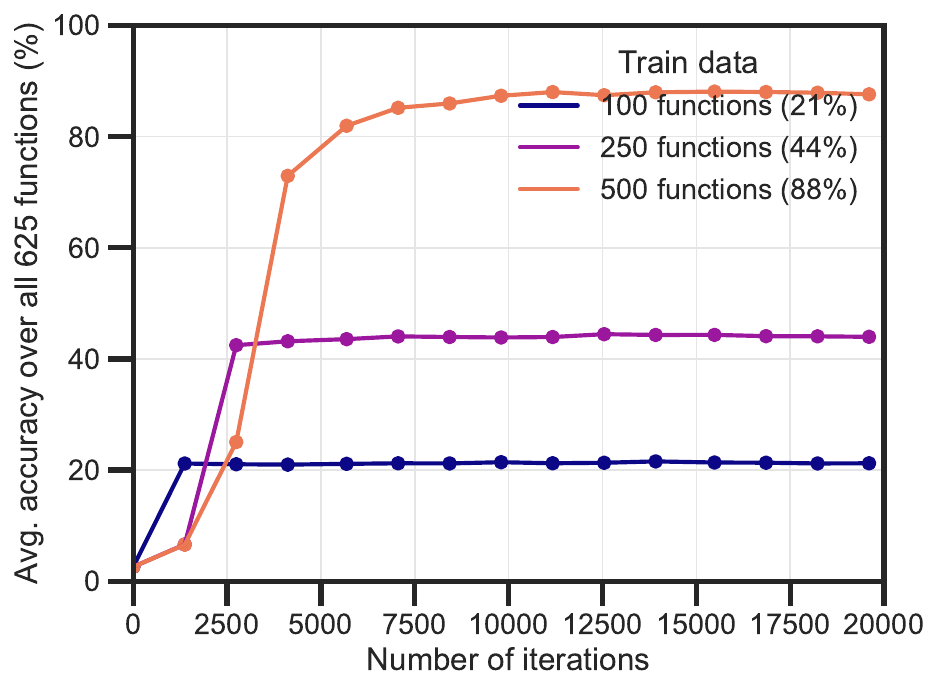}
    \caption{\textbf{Transformers fail to generalize to compositions of even 2 bijections, when trained with the direct prompt format.} The curve depicts the accuracy over all 625 in-order compositions of two bijections (25 choices for each bijection) when trained on different subsets of in-order compositions. The model is trained with direct composition. Even if we train on 500 such compositions, the model fails to generalize to the remaining 125 compositions. This is additional evidence that the model is incapable composing  bijections through direct composition.}
    \label{fig:bijection_direct2}
\end{figure}

In~\Cref{fig:bijectionism}~(Left) we show that Transformers do not learn to compose 5 bijections and only generalize to compositions in the training data. \Cref{fig:bijection_direct2} augments this result and shows that a similar failure occurs even when we consider the composition of just two bijections. Hence the model may not compose some function in the direct prompt format and the step-by-step format with an autoregressive objective is far more amenable to compositions.

\subsection{Additional experiments with training data from \random~and \base}

In this section, we conduct a collection of analyses for a model trained on in-order compositions of 5 bijections in the step-by-step prompt format. We perform the following experiments:
\begin{enumerate*}
    \item compare how \base~and \random~generalize to other in-order compositions (Fig.~\ref{fig:many_and_few}); 
    \item change the number of \random~functions in the training data (Fig.~\ref{fig:no_of_rand}); 
    \item limit the maximum number of compositions in the training data and evaluate compositional generalization (Fig.~\ref{fig:max_compose});
    \item look at alternate evaluation metrics (Fig.~\ref{fig:eval_metric}); and
    \item test if the compositions are systematic~\citep{hupkes2020compositionality} (Fig.~\ref{fig:systematicity}).
\end{enumerate*}

\begin{figure}[H]
    \centering
    \includegraphics[width=0.99\textwidth]{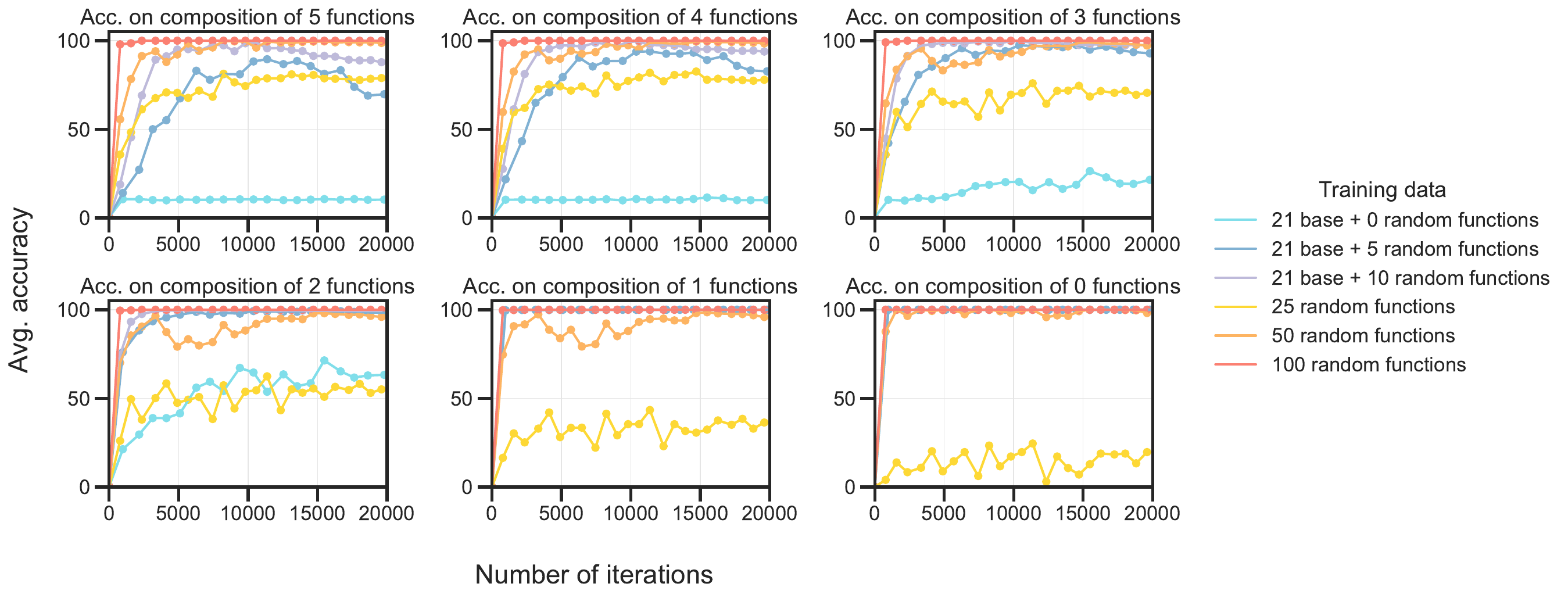}
    \caption{\textbf{How do different training datasets generalize to compositions of many and few functions?} This is a fine-grained version of~\Cref{fig:exponential_bijection}. Model trained on 50 \random~compositions generalizes poorly compositions of small number of functions while a model trained on the \base~generalizes poorly to composition of 4 or 5 functions.}
    \label{fig:many_and_few}
\end{figure}

\begin{figure}[H]
    \centering
    \includegraphics[width=0.68\textwidth]{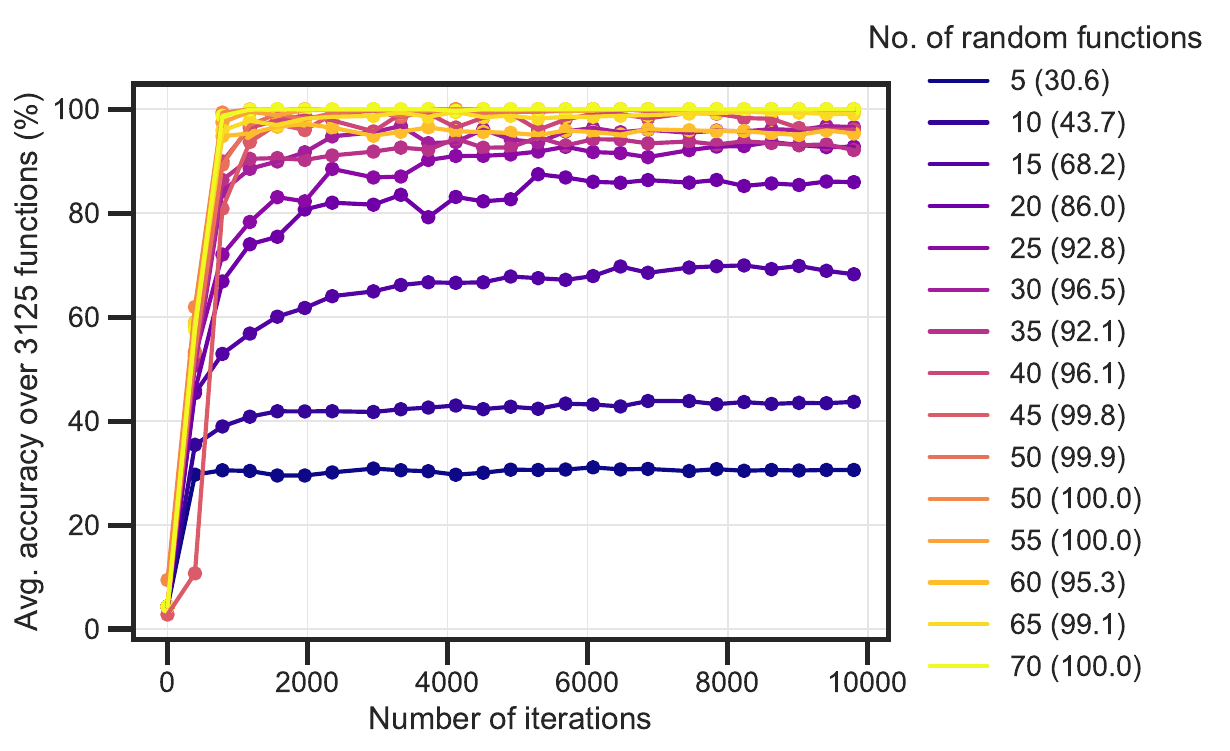}
    \caption{\label{fig:no_of_rand}\textbf{Training with different numbers of \random~ functions.}
        We train on a different number of random functions ranging from 5-70 in steps of 5. These plots are the accuracies averaged over all in-order compositions of 5 bijections over the course of training.}
\end{figure}

\begin{figure}[H]
    \centering
        \includegraphics[width=0.6\textwidth]{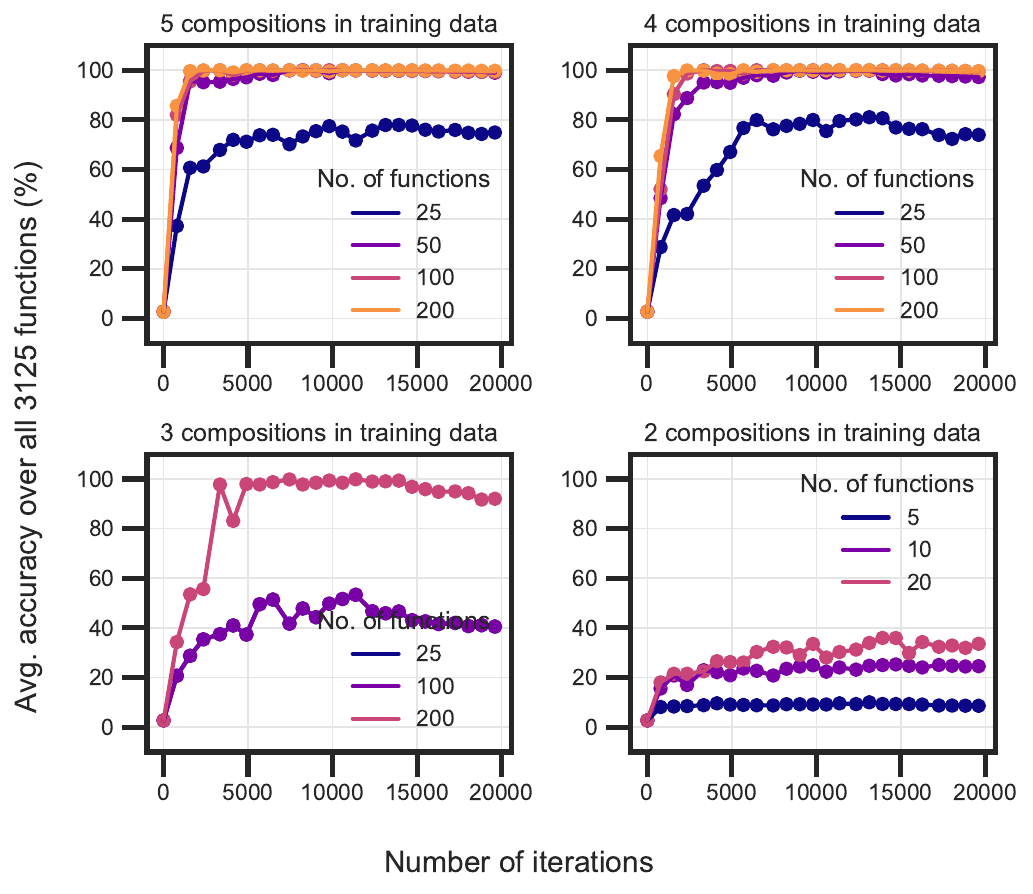}
        \caption{\label{fig:max_compose}\textbf{Limiting maximum number of compositions in the training data.}
            The figure plots the accuracy on all in-order compositions against the number of training iterations.  Each sub-plot considers compositions of size exactly 2, 3, 4, 5, respectively in the training data. The model is able to generalize to most in-order compositions only if the training data consists of compositions of size at least 3 (bottom-right).
        }
\end{figure}

\begin{figure}[H]
    \centering
        \includegraphics[width=0.95\textwidth]{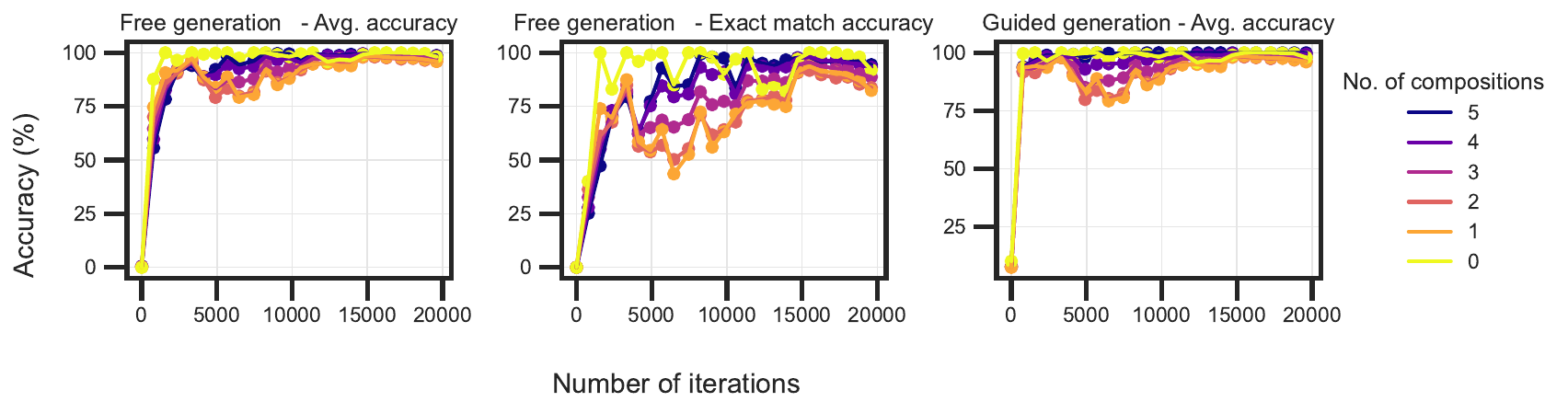}

        \includegraphics[width=0.95\textwidth]{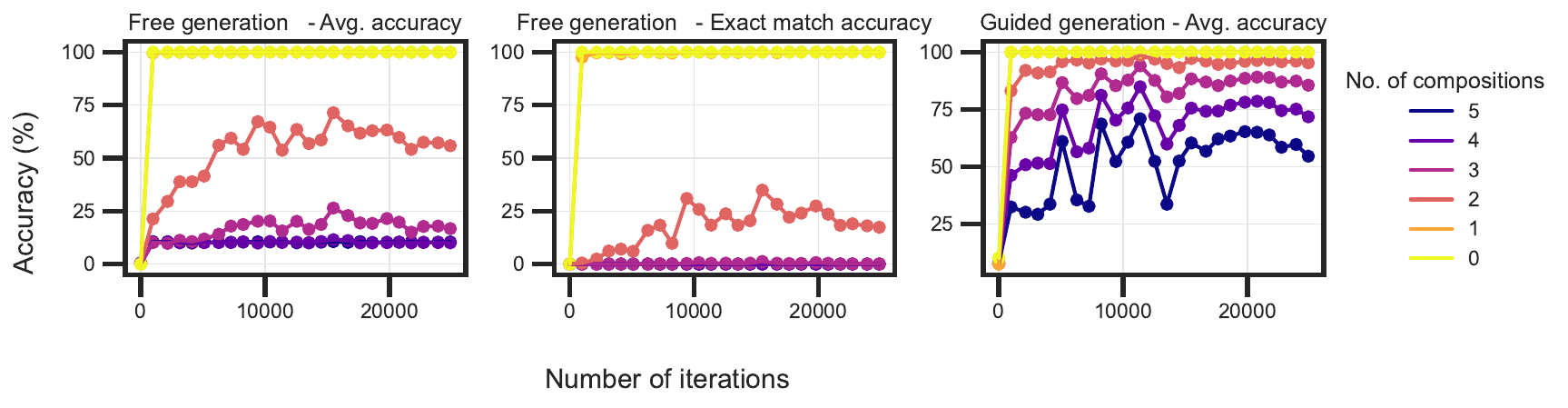}
        \caption{\label{fig:eval_metric}
            \textbf{Evaluation metric.}
            We consider 3 different metrics for evaluating the models. The left column
            considers the average accuracy when the model generates
            \textbf{The choice of metric doesn't change qualitative trends.}
            Each sub-plot considers compositions of only size 2, 3, 4, 5, respectively. In
            each plot, we vary the number of such functions that are present int he
            training data.
            \textbf{One exception is when we train on compositions of size 2.} In this
            case, the guided generation accuracy is high, but the free generation accuracy
            is not.
        }
        \label{fig:app:metrics}
\end{figure}

\begin{figure}[H]
    \centering
    \includegraphics[width=0.75\textwidth]{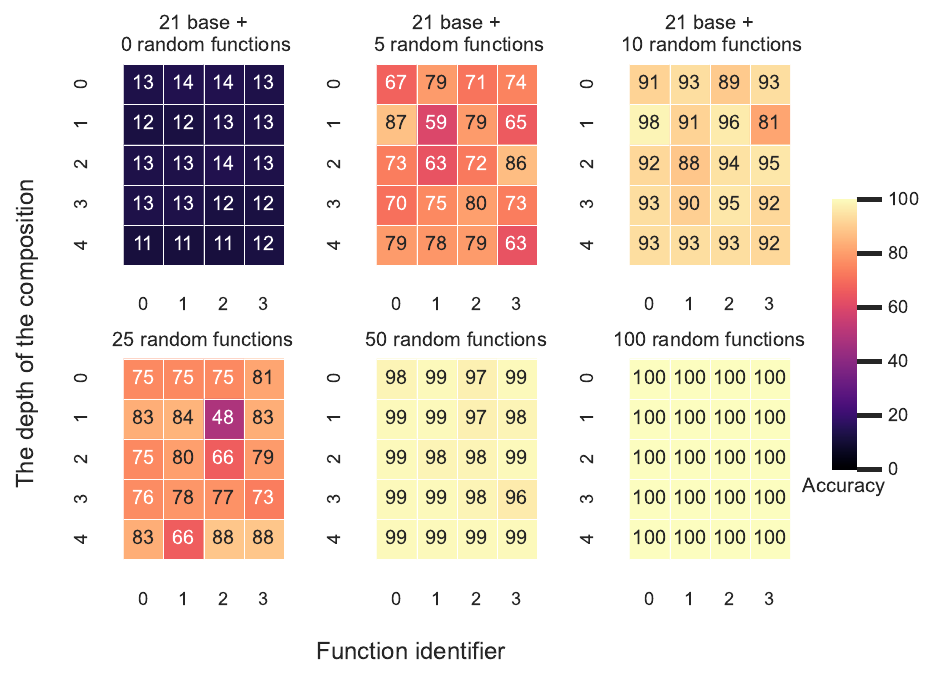}
    \caption{\textbf{Systematicity.} We consider trained models from~\Cref{fig:exponential_bijection} and analyze the accuracy of each of the 20 functions (atomic capabilities) when averaged all instances in which it was used compositionally. We breakdown the results to see if certain functions are more accurate when used in compositions compared to others and find that models seem to learn all functions equally well.}
    \label{fig:systematicity}
\end{figure}

\subsection{Token embeddings}\label{s:app:token_embeddings}

We study the token embeddings of the Transformer models and observe that they are similar for models with different number of layers and attention heads (see Fig.~\ref{fig:word_embds}). We notice a block diagonal structure that separates task tokens from the data tokens. We also observe another block diagonal structure within the task tokens which occurs when we train only on in-order compositions. 

\begin{figure}[!htb]
    \centering
    \includegraphics[width=0.4\textwidth]{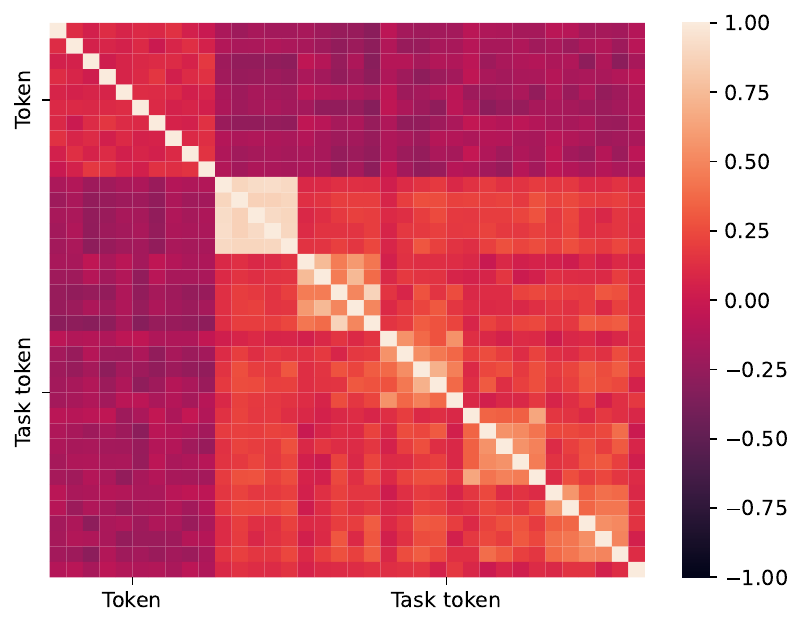}
    \caption{\label{fig:word_embds}\textbf{Word embedding correlations present a block-diagonal structure that separates data tokens from task tokens.} We plot the inner product between all pairs of word embeddings of the tokens. The task tokens are orthogonal to the set of input tokens. Different functions in the same level, i.e. $\{F_i^{(l)}\}_{i=1}^N$ for a fixed $l$, form a block-diagonal in this matrix. We observe similar word embeddings in Transformers of different sizes.}
\end{figure}

\section{Analysis of Step-by-step and Direct Prompt Formats}\label{s:app:theory}

\subsection{Transformers for the step-by-step prompt format}\label{s:app:theory_step}

We prove that there exists Transformers that can compositionally generalize in the step-by-step prompt format. Such a constructive proof, similar to~\citep{von2023transformers,ahn2023transformers,weiss2021thinking,li2023dissecting}, can be used to generate plausible mechanistic hypothesis by highlighting the role of the attention and MLP layers. While the universal approximation theorem suggests that any function can be represented by a wide enough multi-layer perceptron (MLP), the construction suggests that Transformers can represent the same function efficiently.

\paragraph {Description of the data.}
We will operate with a simplified prompt format where a composition of three functions is to be applied to a single input token. The construction can be generalized to compositions of more functions or to multiple input tokens. The input prompt $[x_{F_1}, x_{F_2}, x_{F_3}, x_d]$ has three task tokens and a single data token, and the desired output for this prompt is $[F_1(x_d), F_2 \circ F_1(x_d), F_3 \circ F_2 \circ F_1(x_d)]$. 

The position encodings $P = \begin{bmatrix} p_1 & p_2 & \cdots & p_6 \end{bmatrix}$ are learnable parameters and have dimension $d_p$, i.e., $P \in \mathbb{R}^{d_p \times 6}$. The number of input tokens is $d_v$ and the number of task tokens is $d_f$. Both input tokens $x_d$ and task tokens $x_{F_1}$ are embedded as a one-hot vector in $\mathbb{R}^{d_x}$ where $d_x = d_v + d_f$. The first $d_v$ dimensions are used to embed the data tokens and the last $d_f$ dimensions embed the task token. Henceforth, both $x_d$ and $x_{F_1}$ refer to the corresponding one-hot vectors in $\mathbb{R}^{d_x}$. For convenience, we define $d = d_x + d_p$. Tying this back to to~\cref{sec:setup}, observe that $|X_d| = d_v$ and $|X_f| = d_f$.
We denote the input to the model using $Z$, which includes the token embedding and position encoding. Specifically, we have
\[ Z = 
    \begin{bmatrix}
        x_{F_1} & x_{F_2} & x_{F_3} & x
                & F_1(x_d) & F_2 \circ F_1(x_d)  \\
        p_1 & p_2 & p_3 & p_4 & p_5 & p_6
    \end{bmatrix},
        % x_{F_1} & x_{F_2} & x_{F_3} & x_d
        %         & F_1(x_d) & F_2 \circ F_1(x_d) & F_3 \circ F_2 \circ F_1(x_d) \\
        % p_1 & p_2 & p_3 & p_4 & p_5 & p_6 & p_7
 \]
i.e., $Z \in \mathbb{R}^{d \times 6}$. We assume that the position encoding is concatenated to the token embedding as opposed to added to it.

\paragraph{Matrix notation.}
We use $\mathbbm{1}_{x_d}$ to denote a one-hot vector in the space $\mathbb{R}^{d_v}$, i.e., it excludes dimensions for the task token. On the other hand, $x_d$ denotes a one-hot vector in $\mathbb{R}^{d_x}$. We use $I_{n \times n}$ to denote an identity matrix of size $n \times n$, $1_{m \times n}$ and $0_{m \times n}$ to denote matrices of 1s and 0s of size $m \times n$, and $1_{n}$ and $0_n$ to denote matrices of size $n \times 1$.

\paragraph {Description of the architecture.}
Before describing the Transformer architecture, we first define the attention and MLP layers. We use a simplified parameterization of linear attention~\citep{ahn2023transformers} with weights $Q$ and $K$. The MLP contains two fully connected layers with a ReLU non-linearity parameterized by the weights $W_1$ and $W_2$. The attention and MLP layers are functions of $Z \in \mathbb{R}^{d \times 6}$ and are defined as:
\begin{align*}
    &\mathrm{Attn}_{Q, K}(Z) = (KZ)   (M \odot  Z^T Q Z),\,\, \text{and} \\
    &\mathrm{MLP}_{W_1, W_2}(Z) =   W_2 \mathrm{ReLU}(W_1 Z),
\end{align*}
where $Q, K \in \mathbb{R}^{d \times d}$, $W_1 \in \mathbb{R}^{d \times (d_f d_v)}$ and $W_2 \in \mathbb{R}^{(d_f d_v) \times d}$. The matrix $M \in \mathbb{R}^{6 \times 6}$ enforces causal attention and restricts the attention to inputs from previous time-steps, i.e.,
\[M = \begin{bmatrix}
    1 & 1 & 1 & \cdots & 1 \\
    0 & 1 & 1 & \cdots & 1 \\
    \vdots & \vdots & \vdots & \vdots & \vdots \\
    0 & 0 & 0 & \cdots & 1
\end{bmatrix}.\]

We consider a 1-layer Transformer with an attention layer followed by an MLP layer. We
omit layer-norm to simplify the proofs. The function computed by the Transformer is 
\[ \mathrm{Tr}_{Q, K, W_1, W_2}(Z) = \mathrm{MLP}\left(\mathrm{Attn}(Z) + Z) + \mathrm{Attn}(Z) + Z \right). \]
Henceforth, we omit the subscripts of $\mathrm{Attn}$, $\mathrm{MLP}$ and $\mathrm{Tr}$ for brevity. We include a residual connection after both the attention and MLP layers which mirrors a typical Transformer architecture~\citep{vaswani2017attention}.

The output of the Transformer is passed through an unembedding matrix $W_e$ followed by a
Softmax layer to obtain a probability distribution over the next token denoted by
\[ P(Y|Z) = \mathrm{Softmax}(W_e \mathrm{Tr}(Z)). \] 
\begin{theorem}\label{thm:step}
    There exists weights $P, Q, K, W_1, W_2$ and position encodings $P$  such
    that an Autoregressive Transformer can compositionally generalize to any prompt $[x_{F_1}, x_{F_2}, x_{F_3}, x_d]$. The values of the weights satisfy
    \begin{align*}
        &P^T P = \begin{bmatrix}
            I_{3 \times 3} & I_{3 \times 3} \\
            I_{3 \times 3 } & I_{3 \times 3} \\
        \end{bmatrix}, \qquad
        Q = \begin{bmatrix}
            0_{d \times d} & 0_{d \times d_p}  \\
            0_{d_p \times d} & I_{d_p \times d_p}  \\
        \end{bmatrix}, \qquad 
        K = \begin{bmatrix}
            0_{d_v \times d_v} & 0_{d_f \times d_v} &  0_{d \times d_p}  \\
            0_{d_f \times d_v} & I_{d_f \times d_f} &  0_{d \times d_p}  \\
            0_{d_v \times d} & 0_{d_f \times d} & 0_{d_p \times d_p} \\
        \end{bmatrix}, \\
        &W_1 = \underbrace{\begin{bmatrix}
            \mathbbm{1}_{x_{d_1}}^T & \mathbbm{1}_{x_{d_1}}^T & \mathbbm{1}_{x_{d_1}}^T & \cdots & \mathbbm{1}_{x_{d_1}}^T \\
            \mathbbm{1}_{x_{d_2}}^T & \mathbbm{1}_{x_{d_2}}^T & \mathbbm{1}_{x_{d_2}}^T & \cdots & \mathbbm{1}_{x_{d_2}}^T \\
            \vdots & \vdots & \vdots & \vdots \\
            \mathbbm{1}_{x_{d_v}}^T & \mathbbm{1}_{x_{d_v}}^T & \mathbbm{1}_{x_{d_v}}^T & \cdots & \mathbbm{1}_{x_{d_v}}^T \\
            0_{d_v}^T & -1_{d_v}^T & -1_{d_v}^T & \cdots & - 1_{d_v}^T \\
            - 1_{d_v}^T & 0_{d_v}^T &  - 1_{d_v}^T & \cdots & - 1_{d_v}^T \\
            \vdots & \vdots & \vdots & \vdots \\
            - 1_{d_v}^T & - 1_{d_v}^T & -1_{d_v}^T &  \cdots & 0_{d_v}^T \\
            0_{d_p \times d_v} & 0_{d_p \times d_v} & 0_{d_p \times d_v} & \cdots & 0_{d_p \times d_v} \\
        \end{bmatrix}^T}_{d_f \times d_v \text{columns}}, \text{ and} \qquad 
        W_2 = \begin{bmatrix}
            F_{i_1}(x_{d_1})^T - x_{d_1}^T - x_{F_{i_1}}^T \\
            F_{i_1}(x_{d_2})^T - x_{d_2}^T - x_{F_{i_1}}^T \\
            \vdots \\
            F_{i_1}(x_{d_v})^T - x_{d_v}^T - x_{F_{i_1}}^T \\
            F_{i_2}(x_{d_1})^T - x_{d_1}^T - x_{F_{i_2}}^T \\
            F_{i_2}(x_{d_2})^T - x_{d_2}^T - x_{F_{i_2}}^T \\
            \vdots \\
            F_{i_2}(x_{d_v})^T - x_{d_v}^T - x_{F_{i_2}}^T \\
            \vdots \\
            F_{i_T}(x_{d_1})^T - x_{d_1}^T - x_{F_{i_T}}^T \\
            F_{i_T}(x_{d_2})^T - x_{d_2}^T - x_{F_{i_T}}^T \\
            \vdots \\
            F_T(x_{d_v}) - x_{d_v} - x_{F_{i_T}}
        \end{bmatrix}^T. 
    \end{align*}
\end{theorem}
\begin{proof}
    See~\Cref{s:app:proof_step_by_step}.
\end{proof}

\begin{wrapfigure}{r}{0.5\textwidth}
    \vspace*{-29pt}
    \centering
    \includegraphics[width=0.6\linewidth]{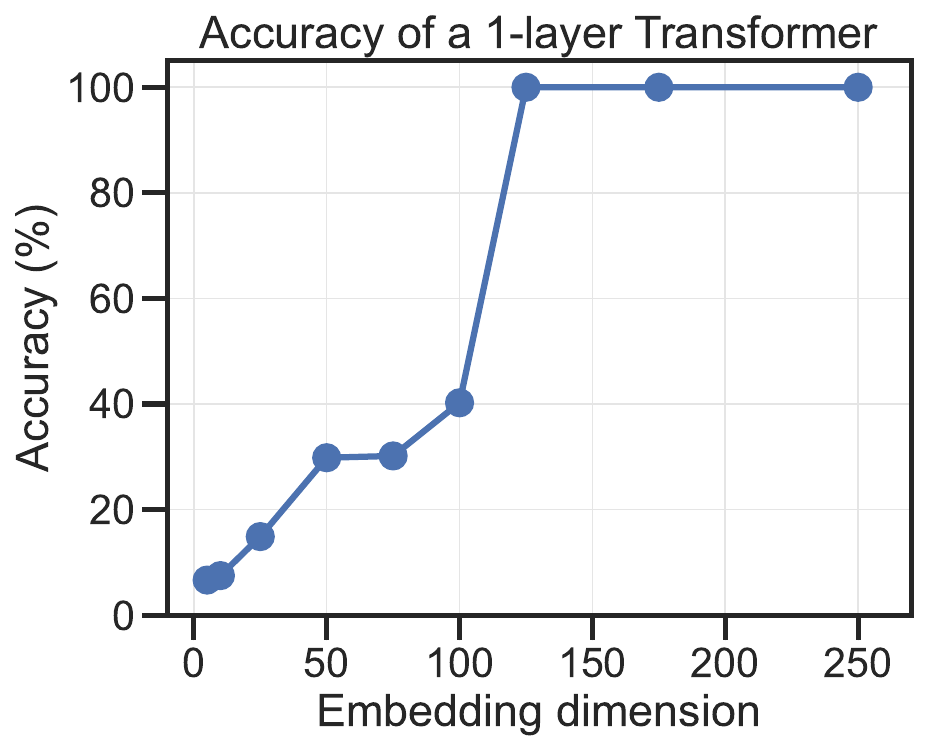}
    \caption{We see a sharp increase in accuracy as we increase the embedding
    dimension of the Transformer. The number of hidden units in the MLP of the Transformer
is 4 times the size of the embedding dimension.}
    \label{fig:app:embed}
    \vspace*{-20pt}
\end{wrapfigure}
The construction uses the attention layer to aggregate the task token and data token, i.e., attention selects the relevant task token. The query vector of the attention selects
the right task using the position encoding. The first layer of the MLP projects the summation of the task and
data tokens (present in orthogonal spaces) onto the Cartesian product of the
set of task and data tokens. The second layer computes the function and acts similar to a lookup table~\citep{geva2022lm}. 

The construction requires the output of the first fully-connected layer has size at least
$d_f d_v$ in order to encode the task and input tokens. In our experiments, we set $d_v =
10$ and $d_f=21$ and hence the number of hidden units must be at least 210. In
practice, we require at least $500$ hidden units (see Fig.~\ref{fig:app:embed}),
which is not too far from our estimate. We conjecture that the additional 
hidden units are helpful for optimization.

\subsection{Transformers for the direct prompt format}\label{s:app:theory_direct}
We also prove the existence of a Transformer for a compositions of bijections in the direct prompt format. Unlike the
step-by-step format, the direct prompt format lacks a ``scratchpad''~\citep{nye2021show} for
the intermediates outputs of the composition. In our construction, we use $K=3$ Transformer
blocks to compute the composition of $K$ functions; the output of the $k$-th block is the
result of the $k^{\text{th}}$ step of the composition.

\paragraph{Description of the data.}
We consider the composition of 3 functions with an input prompt denoted by $[x_{F_1},
x_{F_2}, x_{F_3}, x_d]$. Unlike the step-by-step format, the output is just a single token
$[F_3 \circ F_2 \circ F_1(x_d)]$. The position encodings are denoted by $P = [p_1, p_2, \dots, p_4]$
where $p_i = \begin{bmatrix} p_{i1}^T & p_{i2}^T & p_{i3}^T \end{bmatrix}^T$ and $p_i
\in \mathbb{R}^{d_p}$ and $p_{ij} \in \mathbb{R}_{d_p / 3}$. The dimensions $d_x, d_v, d$
and $d_p$ represent the same quantities. We use $\bar d_p$ to replace $\frac{d_p}{3}$.
The input to the model is
\[ Z = \begin{bmatrix}
    x_{F_1} & x_{F_2} & x_{F_3} & x_d \\
    p_{11} & p_{12} & p_{13} & p_{14} \\
    p_{21} & p_{22} & p_{23} & p_{24} \\
    p_{31} & p_{32} & p_{33} & p_{34} \\
\end{bmatrix},
\] 
where $Z \in \mathbb{R}^{d \times 4}$.

\paragraph{Description of the architecture.}
Each Transformer block is defined similar to the step-by-step format, i.e.,
\[ \mathrm{Block}_{Q_i, K_i, W_{i1}, W_{i2}}(Z) = \mathrm{MLP}_i (\mathrm{Attn}_i(Z) + Z)
    + (\mathrm{Attn}_i(Z) + Z), \] 
which we henceforth denote by $\mathrm{Block}_i(Z)$. Unlike the step-by-step format, the model is now composed of $3$
blocks corresponding to the 3 steps of the compositional task the model is expected to solve, i.e., 
\[ \mathrm{Tr}(Z) = \mathrm{Block}_3(\mathrm{Block}_2(\mathrm{Block}_1(Z))). \]
This input is passed through a Softmax layer to predict a probability distribution over the next
token, denoted by $P(Y \mid Z) = \mathrm{Softmax}(W_e \mathrm{Tr}(Z))$.

\begin{theorem}\label{thm:direct}
    There exist weights $P_i, Q_i, K_i, W_{1i}, W_{2i}$ for $i \in [1, 3]$
    and position encodings $P$  such that the a 3-layer Transformer can compositionally generalize to any prompt of the form $[x_{F_1}, x_{F_2}, x_{F_3},
    x_d]$. The values of the weights satisfy
    \begin{align*}
        &Q_1 = \begin{bmatrix}
            0_{d \times d} & 0_{d \times \bar d_p} & 0_{d \times \bar d_p} & 0_{d \times
            \bar d_p} \\
            0_{\bar d_p \times d} & I_{\bar d_p} & 0_{\bar d_p \times \bar d_p} & 0_{\bar d_p \times \bar d_p} \\
            0_{\bar d_p \times d} & 0_{\bar d_p \times \bar d_p} & 0_{\bar d_p \times \bar
            d_p} & 0_{\bar d_p \times \bar d_p} \\
            0_{\bar d_p \times d} & 0_{\bar d_p \times \bar d_p} & 0_{\bar d_p \times \bar d_p} & 0_{\bar d_p \times \bar d_p} 
        \end{bmatrix}, \qquad 
        Q_2 = \begin{bmatrix}
            0_{d \times d} & 0_{d \times \bar d_p} & 0_{d \times \bar d_p} & 0_{d \times
            \bar d_p} \\
            0_{\bar d_p \times d} & 0_{\bar d_p \times \bar d_p} & 0_{\bar d_p \times \bar d_p} & 0_{\bar d_p \times \bar d_p} \\
            0_{\bar d_p \times d} & 0_{\bar d_p \times \bar d_p} & I_{\bar d_p} & 0_{\bar d_p \times \bar d_p} \\
            0_{\bar d_p \times d} & 0_{\bar d_p \times \bar d_p} & 0_{\bar d_p \times \bar
            d_p} & 0_{\bar d_p \times \bar d_p}
        \end{bmatrix},  \\
        &Q_3 = \begin{bmatrix}
            0_{d \times d} & 0_{d \times \bar d_p} & 0_{d \times \bar d_p} & 0_{d \times
            \bar d_p} \\
            0_{\bar d_p \times d} & 0_{\bar d_p \times \bar d_p} & 0_{\bar d_p \times \bar d_p} & 0_{\bar d_p \times \bar d_p} \\
            0_{\bar d_p \times d} & 0_{\bar d_p \times \bar d_p} & 0_{\bar d_p \times \bar
            d_p} & 0_{\bar d_p \times \bar d_p} \\
            0_{\bar d_p \times d} & 0_{\bar d_p \times \bar d_p} & 0_{\bar d_p \times \bar
            d_p} & I_{\bar d_p}, \quad
        \end{bmatrix}, \qquad 
        K_1 = \begin{bmatrix}
            0_{d_v \times d_v} & 0_{d_f \times d_v} &  0_{d \times d_p}  \\
            0_{d_f \times d_v} & I_{d_f} &  0_{d \times d_p}  \\
            0_{d_v \times d} & 0_{d_f \times d} & 0_{d_p \times d_p} \\
        \end{bmatrix}, \\
        &\qquad \qquad \qquad \qquad \qquad \qquad K_2 = \frac{K_1}{2}, \qquad
        K_3 = \frac{K_1}{3}, \\
    \end{align*}
    \begin{align*}
        &P_1^T P_1 =\begin{bmatrix}
            1 & 0 & 0 & 1 \\
            0 & 1 & 0 & 0 \\
            0 & 0 & 1 & 0 \\
            1 & 0 & 0 & 1 
        \end{bmatrix}, \qquad 
        P_2^T P_2 = \begin{bmatrix}
            1 & 0 & 0 & 0 \\
            0 & 1 & 0 & 1 \\
            0 & 0 & 1 & 0 \\
            0 & 1 & 0 & 1 \\
        \end{bmatrix}, \qquad
        P_3^T P_3 =  \begin{bmatrix}
            1 & 0 & 0 & 0 \\
            0 & 1 & 0 & 0 \\
            0 & 0 & 1 & 1 \\
            0 & 0 & 1 & 1
        \end{bmatrix},
    \end{align*}
    \begin{align*}
        &W_{11} = \underbrace{\begin{bmatrix}
            \mathbbm{1}_{x_{d_1}}^T & \mathbbm{1}_{x_{d_1}}^T & \mathbbm{1}_{x_{d_1}}^T & \cdots & \mathbbm{1}_{x_{d_1}}^T \\
            \mathbbm{1}_{x_{d_2}}^T & \mathbbm{1}_{x_{d_2}}^T & \mathbbm{1}_{x_{d_2}}^T & \cdots & \mathbbm{1}_{x_{d_2}}^T \\
            \vdots & \vdots & \vdots & \vdots \\
            \mathbbm{1}_{x_{d_v}}^T & \mathbbm{1}_{x_{d_v}}^T & \mathbbm{1}_{x_{d_v}}^T & \cdots & \mathbbm{1}_{x_{d_v}}^T \\
            0_{1 \times d_v} & -1_{1 \times d_v} & -1_{1 \times d_v} & \cdots & - 1_{1 \times d_v} \\
            - 1_{1 \times d_v} & 0_{1 \times d_v} &  - 1_{1 \times d_v} & \cdots & - 1_{1 \times d_v} \\
            \vdots & \vdots & \vdots & \vdots \\
            - 1_{1 \times d_v} & - 1_{1 \times d_v} & -1_{1 \times d_v} &  \cdots & 0_{1 \times d_v} \\
            0_{d_p \times d_v} & 0_{d_p \times d_v} & 0_{d_p \times d_v} & \cdots & 0_{d_p \times d_v} \\
        \end{bmatrix}^T}_{d_f \times d_v \text{columns}}, \qquad 
        W_{12} = \begin{bmatrix}
            F_{i_1}(x_{d_1})^T - x_{d_1}^T - x_{F_{i_1}}^T \\
            F_{i_1}(x_{d_2})^T - x_{d_2}^T - x_{F_{i_1}}^T \\
            \vdots \\
            F_{i_1}(x_{d_v})^T - x_{d_v}^T - x_{F_{i_1}}^T \\
            F_{i_2}(x_{d_1})^T - x_{d_1}^T - x_{F_{i_2}}^T \\
            F_{i_2}(x_{d_2})^T - x_{d_2}^T - x_{F_{i_2}}^T \\
            \vdots \\
            F_{i_2}(x_{d_v})^T - x_{d_v}^T - x_{F_{i_2}}^T \\
            \vdots \\
            F_{i_T}(x_{d_1})^T - x_{d_1}^T - x_{F_{i_T}}^T \\
            F_{i_T}(x_{d_2})^T - x_{d_2}^T - x_{F_{i_T}}^T \\
            \vdots \\
            F_T(x_{d_v}) - x_{d_v} - x_{F_{i_T}}.
        \end{bmatrix}^T  \\
        &\qquad \qquad \qquad \qquad \qquad \qquad 
        W_{21} = W_{22} = W_{23}, \text{ and} \qquad 
        W_{31} = W_{32} = W_{33}.
    \end{align*}
\end{theorem}
\begin{proof}
    See~\Cref{s:app:proof_direct}.
\end{proof}

\subsection{Difference between the direct and step-by-step prompt formats}\label{s:app:step_vs_direct}

The ability to run multiple forward passes through the Transformer allows us tackle a richer class of problems~\citep{merrill2023expresssive}. This ability differentiates the step-by-step and direct prompt formats. In the step-by-step prompt format, the Transformer makes $L$ different forward passes, while the direct prompt format allows only one forward pass through the model to generate the output.  This is also mirrored in our constructions in~\cref{s:app:theory_step,s:app:theory_direct}---a model for the step-by-step prompt format requires only 1 layer, while one for the direct prompt format uses $L=3$ layers to compensate for the lack of multiple forward passes. We expect that a Transformer for the direct prompt format cannot circumvent these computations and conjecture that our Transformer construction for the direct format (in~\cref{s:app:proof_direct}) is efficient with respect to the number of layers.

\begin{conjecture}
We conjecture that a Transformer with width of $\poly(|\mathcal{F}|)$,  needs $\mathcal{O}(L)$ layers in the direct prompt format compared to the $\mathcal{O}(1)$ layers step-by-step format in order to compositionally generalize on our synthetic task.  
\end{conjecture}
% \textit{This leads us to conjecture that a Transformer with fixed width needs $\mathcal{O}(L)$ layers in the direct prompt format compared to the $\mathcal{O}(1)$ layers step-by-step format in order to compositionally generalize on our synthetic task.}

That is, a model must compute all $L$ intermediate outputs of the composition across different layers of the Transformer.   We expand on this further in the next subsection. We also note that as per the universal approximation theorem, it is certainly possible to construct a Transformer with 1-layer such that it generalizes for the direct prompt format; \textit{however, such a model must have its width to be exponential in $|\mathcal{F}|$ in order to store $|\mathcal{F}|^L$ different functions in a single layer.}

\subsubsection{How many training compositions does each prompt format need?}
To further understand the difference between the two prompt formats, we will use a (highly simplified) model to reason about the number of function compositions in the training data that is required for perfect compositional generalization on our task. Let us consider a composition of $L$ of functions from $\mathcal{F}$. We assume that the compositions in the training data $\mathcal{F}_{\text{train}}^L \subset \mathcal{F}^L$ are sampled uniformly at random from the set of all compositions. 

For this analysis, we assume that the Transformer can perfectly identify which functions to compose---which we ascribe to the attention layers---and will focus entirely on capability acquisition which we hypothesize is carried out by the MLP layers. We assume that a Transformer for the step-by-step prompt format must learn a function (capability) only once, while a Transformer for the direct prompt format must learn the function $L$ different times---once for each layer of the Transformer. If the function composition $F^{(l)} \in \mathcal{F}_{\text{train}}^L$ occurs in the training data, we assume that the Transformer for the step-by-step format has learned all the capabilities $F^{(l)}_i \in F^{(l)}$ for $i \in [1, L]$, while a Transformer for the direct prompt format can only learn capability $F^{(l)}_i$ at layer $i$. These assumptions are informed by~\cref{thm:step,thm:direct}.

\paragraph{Detour into the coupon collector's problem.}
In order to learn all $F = |\mathcal{F}|$ capabilities, the training data must contain each capability at least once. We note that this is the coupon collector's problem~\citep{myers2006some}: the collector seeks all distinct coupons and recieves a coupon at every round drawn uniformly at random.  The number of rounds corresponds to the number of function compositions in the training data and we would like to calculate the expected number of rounds required to learn all capabilities. It is a well known result that the expected number of rounds to collect all $F$ coupons is $F H_F$ where $H_F$ is the Harmonic number; asymptotically this is $\mathcal{O}(F \log F)$. Furthermore, the probability that we complete a collection of size $f$, in $n$ rounds is
\[ p(L, f) = \frac{F!}{F^L} \stirling{F-1}{L-1}, \]
where $\stirling{F-1}{K-1}$ is the Stirling number of the second kind. 

In the step-by-step prompt format, we observe $L$ capabilities (or coupons) with every composition. All capabilities are learned if we observe each of them in at least one training sample. The expected number of training compositions $N$ required to learn all capabilities is $\mathcal{O}(\frac{F \log F}{L})$ (see~\citet{Xu_Tang_2011}). On the other hand, the direct prompt format can be treated as $L$ independent coupon collector problems and must observe each capability once for each of the $L$ layers. The expected number of rounds to learn all capabilities is the is the expected value of maximum number of rounds for $L$ indepedent coupon collector problems. If we apply Chebyshev's inequality, we get
\[ P( N \geq F H_F + c \log F) \leq \frac{\pi^2}{6c^2 \log^2 F},   \]
since the variance of $N$ is upper bounded by $\frac{n^2 \pi^2}{6}$. Hence, the  maximum value of $L$ different runs is $\mathcal{O}(F \log F)$ as $n \rightarrow \infty$, or in other words, the expected number of rounds to learn all capabilities is $\mathcal{O}(F \log F)$. The expected number of training compositions differ by a factor of $L$ between the two prompt formats, which tallies with the observation that a Transformer is expected to learn the same set of capabilities $L$ different times in the direct format.

In practice, we find that Transformers for the direct format can sometimes fail to compositionally generalize, even with a large number of compositions in the training data (\Cref{ss:prompting}). We hypothesize that this is attributable to the optimization landscape, i.e., gradient descent is unable to find weights that compositionally generalize and instead prefers weights that memorize compositions of functions present in the training data. In the direct prompt, gradient descent must recover the individual capabilities from a set of compositions of bijections and this is a computationally hard problem since it is similar to finding the minimal generating set of a group (its time complexity is linear in the size of the group which is $\mathcal{O}(F^L)$).

% https://mat.uab.cat/matmat_antiga/PDFv2014/v2014n02.pdf

%
\subsection{Proof of \Cref{thm:step}}\label{s:app:proof_step_by_step}
\paragraph{Step 1: Computing the attention layer.}
The attention layer copies the task tokens onto the relevant data token similar to an
induction head~\citep{olsson2022context}. We first compute the query and value matrices of the attention.
\begin{align*}
    Z^T Q Z &= 
    \begin{bmatrix}
        x_{F_1}^T & p_1^T \\ 
        x_{F_2}^T & p_2^T \\ 
        x_{F_3}^T & p_3^T \\
        x_d^T & p_4^T \\
        F_1(x_d)^T & p_5\\
        F_2 \circ F_1(x_d)^T & p_6^T
    \end{bmatrix}
    \begin{bmatrix}
            0_{d \times d} & 0_{d \times d_p}  \\
            0_{d_p \times d} & I_{d_p \times d_p}  \\
    \end{bmatrix}
    \begin{bmatrix}
        x_{F_1} & x_{F_2} & x_{F_3} & \cdots 
                & F_2 \circ F_1(x_d) \\
        p_1 & p_2 & p_3 & \cdots & p_6
    \end{bmatrix} \\
    &= \begin{bmatrix}
        0 & p_1^T \\ 
        0 & p_2^T \\ 
        0 & p_3^T \\
        0 & p_4^T \\
        0 & p_5^T \\
        0 & p_6^T
    \end{bmatrix}
    \begin{bmatrix}
        x_{F_1} & x_{F_2} & x_{F_3} & \cdots 
                & F_2 \circ F_1(x_d) \\
        p_1 & p_2 & p_3 & \cdots & p_6
    \end{bmatrix} = P^T P 
\end{align*}
Our construction considers a $P$ such that $p_i = p_{i+4}$ for all $i \in [1, 3]$ and $p_i \cdot p_j = 0$ for all $j \in [1, 3]$ and $j \neq i$. The mask $M$ converts $P^T P$ into an upper triangular matrix, and zeroes out all entries in the lower triangle of the matrix.
\[ M \odot (Z^T Q Z) = M \odot (P^T P) =
    M \odot \begin{bmatrix}
            I_{3 \times 3} & I_{3 \times 3} \\
            I_{3 \times 3} & I_{3 \times 3} 
        \end{bmatrix} = 
    \begin{bmatrix}
        I_{3 \times 3} & I_{3 \times 3} \\
        0_{3 \times 3} & I_{3 \times 3}
    \end{bmatrix}
\]
The attention layer computes
\begin{align*}
    \mathrm{Attn}(Z) 
    &= (KZ) (M \odot Z^T Q Z) \\
    &= (KZ) (M \odot P P^T) \\
    &= 
    \begin{bmatrix}
        0_{d_v \times d_v} & 0_{d_f \times d_v} &  0_{d \times d_p}  \\
        0_{d_f \times d_v} & I_{d_f \times d_f} &  0_{d \times d_p}  \\
        0_{d_v \times d} & 0_{d_f \times d} & 0_{d_p \times d_p} \\
    \end{bmatrix} 
    \begin{bmatrix}
        x_{F_1} & x_{F_2} & x_{F_3} & \cdots 
                & F_2 \circ F_1(x_d) \\
        p_1 & p_2 & p_3 & \cdots & p_6
    \end{bmatrix}
    \begin{bmatrix}
            I_{3 \times 3} & I_{3 \times 3} \\
            0_{3 \times 3} & I_{3 \times 3} \\
     \end{bmatrix} \\
    &= \begin{bmatrix}
        x_{F_1} & x_{F_2} & x_{F_3} & 0_{d} & 0_{d} & 0_{d} \\
        0_{d_p} & 0_{d_p} & 0_{d_p} & 0_{d_p} & 0_{d_p} & 0_{d_p} \\
    \end{bmatrix}
    \begin{bmatrix}
            I_{3 \times 3} & I_{3 \times 3} \\
            0_{3 \times 3} & I_{3 \times 3} \\
     \end{bmatrix} \\
    &= \begin{bmatrix}
        x_{F_1} & x_{F_2} & x_{F_3} & x_{F_1} & x_{F_2} & x_{F_3} \\
        0_{d_p} & 0_{d_p} & 0_{d_p} & 0_{d_p} & 0_{d_p} & 0_{d_p}
     \end{bmatrix}
\end{align*}
which when added to $Z$ yields
\[ \mathrm{Attn}(Z) + Z = 
    \begin{bmatrix}
        2 x_{F_1} & 2 x_{F_2} & 2 x_{F_3} & x_d + x_{F_1}
                  & F_1(x_d) + x_{F_2} & F_2 \circ F_1(x_d) + x_{F_3} \\
        p_1 & p_2 & p_3 & p_4 & p_5 & p_6,
    \end{bmatrix} \]
if we also include the residual stream to the output of the attention layer.

\paragraph{Step 2: Computing the MLP layer.}
After the attention layer, the data and task tokens are aggregated at one location in orthogonal sub-spaces. The MLP uses the task and data token to compute the function. The first fully-connected layer projects the input $\mathbb{R}_{d_v d_t}$, which uniquely identifies the task and data tokens which is used to retrived the function from $W_2$. The first fully-connected layer computes
\begin{align*}
    (\mathrm{Attn}(Z) + Z)^T W_1^T &=
    \begin{bmatrix}
        2 x_{F_1}^T & p_1^T \\
        2 x_{F_2}^T & p_2^T \\
        2 x_{F_3}^T & p_3^T \\
        x_d^T + x_{F_1}^T & p_4^T \\
        F_1(x_d)^T + x_{F_2}^T & p_5^T \\
        F_2(F_1(x_d))^T + x_{F_3}^T & p_6^T
    \end{bmatrix}
    \begin{bmatrix}
            \mathbbm{1}_{x_{d_1}}^T & \mathbbm{1}_{x_{d_1}}^T & \mathbbm{1}_{x_{d_1}}^T & \cdots & \mathbbm{1}_{x_{d_1}}^T \\
            \mathbbm{1}_{x_{d_2}}^T & \mathbbm{1}_{x_{d_2}}^T & \mathbbm{1}_{x_{d_2}}^T & \cdots & \mathbbm{1}_{x_{d_2}}^T \\
            \vdots & \vdots & \ddots & \vdots \\
            \mathbbm{1}_{x_{d_v}}^T & \mathbbm{1}_{x_{d_v}}^T & \mathbbm{1}_{x_{d_v}}^T & \cdots & \mathbbm{1}_{x_{d_v}}^T \\
            0_{d_v}^T & -1_{d_v}^T & -1_{d_v}^T & \cdots & - 1_{d_v}^T \\
            - 1_{d_v}^T & 0_{d_v}^T &  - 1_{d_v}^T & \cdots & - 1_{d_v}^T \\
            \vdots & \vdots & \ddots & \vdots \\
            - 1_{d_v}^T & - 1_{d_v}^T & -1_{d_v}^T &  \cdots & 0_{d_v}^T \\
            0_{d_p \times d_v} & 0_{d_p \times d_v} & 0_{d_p \times d_v} & \cdots & 0_{d_p \times d_v} \\
        \end{bmatrix} \\
        &= \begin{bmatrix}
            -2_{d_v}^T & \cdots & \cdots & 0_{d_v}^T & \cdots & \cdots & -2_{d_v}^T \\
            -2_{d_v}^T & \cdots & 0_{d_v}^T & \cdots & \cdots & \cdots & -2_{d_v}^T \\
            -2_{d_v}^T & \cdots & \cdots & \cdots & 0_{d_v}^T & \cdots & -2_{d_v}^T \\
            -1_{d_v}^T + \mathbbm{1}_{x_d}^T & \cdots & \cdots & \mathbbm{1}_{x_d}^T & \cdots & \cdots & -1_{d_v}^T + \mathbbm{1}_{x_d}^T\\
            -1_{d_v}^T + \mathbbm{1}_{F_1(x_d)}^T   & \cdots &  \mathbbm{1}_{F_1(x_d)}^T & \cdots & \cdots & \cdots & -1_{d_v}^T + \mathbbm{1}_{F_1(x_d)}^T \\
            -1_{d_v}^T + \mathbbm{1}_{F_2 \circ F_1(x_d)}^T & \cdots &  & \cdots & \mathbbm{1}_{F_2 \circ F_1(x_d)}^T & \cdots & -1_{d_v}^T + \mathbbm{1}_{F_2 \circ F_1(x_d)}^T \\
        \end{bmatrix}
\end{align*}
The above matrix has $d_f d_v$ columns represented as $d_f$ blocks of size $d_v$. The $0$ matrix in the first, second and third row occupy $d_v$ columns each. In particular, they occupy the blocks $j_1$, $j_2$ and $j_3$ where $F_i = F_{i_{j_i}}$, i.e. the block number corresponds to index in the one-hot representation of the task tokens. Let $\mathbbm{1}_{(x, F)}$ denote a one-hot vector in $\mathbb{R}^{d_v \times d_f}$, i.e., it is a one-hot vector that uniquely identifies the task and data token. We can succincintly express the output after the non-linearity as follows:
\begin{align*}
    \mathrm{ReLU}(W_1 (\mathrm{Attn}(Z) + Z))
    &= \mathrm{ReLU}((\mathrm{Attn}(Z) + Z)^T W_1^T)^T)  \\
        &= \begin{bmatrix}
            0_{d_v} & 0_{d_v}  & 0_{d_v}  & 0_{d_v} & 0_{d_v}  & 0_{d_v}  \\
            0_{d_v} & 0_{d_v}  & 0_{d_v}  & \cdots   & \cdots   & \cdots  \\
            0_{d_v} & 0_{d_v}  & 0_{d_v}  & \cdots  & \mathbbm{1}_{F_1(x_d)}  & \cdots \\
            0_{d_v} & 0_{d_v}  & 0_{d_v}  & \mathbbm{1}_{x_d}  & \cdots  & \cdots  \\
            0_{d_v} & 0_{d_v}  & 0_{d_v}  & \cdots  & \cdots  & \mathbbm{1}_{F_2 \circ F_1(x_d)}  \\
            \vdots & \vdots & \vdots & \vdots & \vdots & \vdots \\
            0_{d_v} & 0_{d_v}  & 0_{d_v}  & 0_{d_v} & 0_{d_v}  & 0_{d_v}  \\
        \end{bmatrix} \\
    &= \begin{bmatrix}
        0_{d_v d_f} &
        0_{d_v d_f} &
        0_{d_v d_f} &
        \mathbbm{1}_{(x_{d}, F_1)} &
        \mathbbm{1}_{(F_1(x_{d}), F_2)} &
        \mathbbm{1}_{(F_2 \circ F_1(x_{d}), F_3)}
    \end{bmatrix}
\end{align*}

Including the final weight matrix $W_2$, we get
\begin{align*}
    W_2 \mathrm{ReLU}(W_1 (\mathrm{Attn}(Z) + Z)) 
    &= W_2
    \begin{bmatrix}
        0_{d_v d_f} &
        0_{d_v d_f} &
        0_{d_v d_f} &
        \mathbbm{1}_{(x_{d}, F_1)} &
        \mathbbm{1}_{(F_1(x_{d}), F_2)} &
        \mathbbm{1}_{(F_2 \circ F_1(x_{d}), F_3)}
    \end{bmatrix} \\
    &= \begin{bmatrix}
        0_{d}^T & 0_{d_p}^T\\
        0_{d}^T & 0_{d_p}^T\\
        0_{d}^T & 0_{d_p}^T\\
        F_1(x_d)^T - x_d - x_{F_1} & 0_{d_p}^T\\
        F_2\circ F_1(x_d) - x_{F_1(x_d)} - x_{F_2} & 0_{d_p}^T\\
        F_3\circ F_2 \circ F_1 (x_d) - x_{F_2 \circ F_1(x_d)} - x_{F_3} & 0_{d_p}^T\\
    \end{bmatrix}^T 
\end{align*}
Hence, the output of the Transformer is
\begin{align*}
    \mathrm{Tr}(Z) 
    &= \mathrm{MLP}(\mathrm{Attn}(Z) + Z) + (\mathrm{Attn}(Z) + Z) \\
    &= \begin{bmatrix}
        0_{d}^T & 0_{d_p}^T\\
        0_{d}^T & 0_{d_p}^T\\
        0_{d}^T & 0_{d_p}^T\\
        F_1(x_d)^T - x_d^T - x_{F_1}^T & 0_{d_p}^T \\
        F_2\circ F_1(x_d)^T - x_{F_1(x_d)}^T - x_{F_2}^T & 0_{d_p}^T\\
        F_3\circ F_2 \circ F_1 (x_d)^T - x_{F_2 \circ F_1(x_d)}^T - x_{F_3}^T & 0_{d_p}^T\\
    \end{bmatrix}^T +
    \begin{bmatrix}
        2 x_{F_1}^T & p_1^T \\
        2 x_{F_2}^T & p_2^T \\
        2 x_{F_3}^T & p_3^T \\
        x_d^T + x_{F_1}^T & p_4^T \\
        x_{F_1(x_d)}^T + x_{F_2}^T & p_5^T \\
        x_{F_2 \circ F_1(x_d)}^T + x_{F_3}^T & p_6^T \\
    \end{bmatrix}^T \\
    &= \begin{bmatrix}
        2 x_{F_1} & 2 x_{F_2} & 2 x_{F_3} & F_1(x_d)
                  & F_2 \circ F_1(x_d)  & F_3 \circ F_2 \circ F_1(x_d) \\
        p_1 & p_2 & p_3 & p_4 & p_5 & p_6
    \end{bmatrix} \tag{\theequation}\label{eq:tr}
\end{align*}

If we set 
\[ W_e = \begin{bmatrix}
            I_{d \times d} & 0_{d \times d_p}  \\
            0_{d_p \times d} & o_{d_p \times d_p}  \\
    \end{bmatrix},
\]
then $W_e \mathrm{Tr}(Z)$ evaluates to
\[ \begin{bmatrix}
    2 x_{F_1} & 2 x_{F_2} & 2 x_{F_3} & F_1(x_d)
              & F_2 \circ F_1(x_d)  & F_3 \circ F_2 \circ F_1(x_d) \\
\end{bmatrix} \]
which will assign high probabilities to the desired token when passed through a Softmax
layer. Hence, a Transformer prompted with $[x_{F_1}, x_{F_2}, x_{F_3}, x_d]$ will auto-regressively generate
$[F_1(x_d), F_2 \circ F_1(x_d), F_3 \circ F_2 \circ F_1(x_d)]$ for any combination of data and task tokens.
\qed

\subsection{Proof of \Cref{thm:direct}}\label{s:app:proof_direct}

The details of construction are similar to~\Cref{s:app:proof_step_by_step}.

\paragraph{Step 1: Computing the output of the first block.}
The first Transformer block computes the first step of the composition. The attention layer in particular, copies the relevant task token to the data token. The value and query matrices of the attention layer in the first Transformer block are
\begin{align*}
    Z^T Q_1 Z &= 
    \begin{bmatrix}
        x_{F_1}^T & p_{11}^T & p_{21}^T & p_{31}^T \\ 
        x_{F_2}^T & p_{12}^T & p_{22}^T & p_{32}^T \\
        x_{F_3}^T & p_{13}^T & p_{23}^T & p_{33}^T \\
        x_d^T & p_{14}^T & p_{24}^T & p_{34}^T
    \end{bmatrix}
    \begin{bmatrix}
            0_{d \times d} & 0_{d \times \bar d_p} & 0_{d \times \bar d_p} & 0_{d \times
            \bar d_p} \\
            0_{\bar d_p \times d} & I_{\bar d_p \times \bar d_p} & 0_{\bar d_p \times \bar d_p} & 0_{\bar d_p \times \bar d_p} \\
            0_{\bar d_p \times d} & 0_{\bar d_p \times \bar d_p} & 0_{\bar d_p \times \bar
            d_p} & 0_{\bar d_p \times \bar d_p} \\
            0_{\bar d_p \times d} & 0_{\bar d_p \times \bar d_p} & 0_{\bar d_p \times \bar d_p} & 0_{\bar d_p \times \bar d_p} 
    \end{bmatrix}
    \begin{bmatrix}
        x_{F_1} & x_{F_2} & x_{F_3} & x_d \\
        p_{11} & p_{12} & p_{13} & p_{14} \\
        p_{21} & p_{22} & p_{23} & p_{24} \\
        p_{31} & p_{32} & p_{33} & p_{34}
    \end{bmatrix} \\
              &= P_1^T P_1,
\end{align*}
and 
\[ 
    K_1 Z  = \begin{bmatrix}
            0_{d_v \times d_v} & 0_{d_f \times d_v} &  0_{d \times d_p}  \\
            0_{d_f \times d_v} & I_{d_f} &  0_{d \times d_p}  \\
            0_{d_v \times d} & 0_{d_f \times d} & 0_{d_p \times d_p} \\
        \end{bmatrix}
    \begin{bmatrix}
        x_{F_1} & x_{F_2} & x_{F_3} & x_d \\
        p_{1} & p_{2} & p_{3} & p_{4}
    \end{bmatrix}
    = \begin{bmatrix}
        x_{F_1} & x_{F_2} & x_{F_3} & 0_{d} \\
        0_{d_p} & 0_{d_p} & 0_{d_p} & 0_{d_p} \\
    \end{bmatrix}
\]
Using the above, the output of the first attention layer added to the residual stream is
\begin{align*}
    \mathrm{Attn}_1(Z) + Z 
    &=  (K_1 Z) (M \odot Z^T Q_1 Z) + Z \\
    &=  (K_1 Z) (M \odot P_1^T P_1) + Z \\
    &= \begin{bmatrix}
        x_{F_1} & x_{F_2} & x_{F_3} & 0 \\
        0_{d_p} & 0_{d_p} & 0_{d_p} & 0_{d_p}
    \end{bmatrix}
    \begin{bmatrix}
            1 & 0 & 0 & 1 \\
            0 & 1 & 0 & 0 \\
            0 & 0 & 1 & 0 \\
            0 & 0 & 0 & 1 
    \end{bmatrix} + Z \\
    &= \begin{bmatrix}
        2x_{F_1} & 2x_{F_2} & 2x_{F_3} & x_d + x_{F_1}\\
        p_{1} & p_{2} & p_{3} & p_{4} \\
    \end{bmatrix}
\end{align*}

Note that $W_{11}$ and $W_{21}$ are identical to $W_1$ and $W_2$ in~\Cref{eq:tr}, and performing a similar calculation yields
\begin{align*}
    \mathrm{Block}_1(Z) 
    &= W_{21} \mathrm{ReLU}(W_{11} (\mathrm{Attn}_1(Z) + Z)) + (\mathrm{Attn}_1(Z) + Z) \\
    &= \begin{bmatrix}
        2 x_{F_1} & 2 x_{F_2} & 2 x_{F_3} & F_1(x_d) \\
        p_1 & p_2 & p_3 & p_4 \\
    \end{bmatrix}
    = Z_{B_1}.
\end{align*}
We denote the output of the first Transformer block by $Z_{B_1}$.

\paragraph{Step 2: Computing the output of the second block.}
The second block uses the output of the first Transformer block to compute the second step of the composition.
We start similarly by computing the query and value matrices of the attention layer, i.e.,
\begin{align*}
    Z_{B_1}^T &Q_2 Z_{B_1} =  \\
    &=\begin{bmatrix}
        2x_{F_1}^T & p_{11}^T & p_{21}^T & p_{31}^T \\ 
        2x_{F_2}^T & p_{12}^T & p_{22}^T & p_{32}^T \\
        2x_{F_3}^T & p_{13}^T & p_{23}^T & p_{33}^T \\
        F_1(x_d)^T & p_{14}^T & p_{24}^T & p_{34}^T
    \end{bmatrix}
    \begin{bmatrix}
            0_{d \times d} & 0_{d \times \bar d_p} & 0_{d \times \bar d_p} & 0_{d \times
            \bar d_p} \\
            0_{\bar d_p \times d} & 0_{\bar d_p \times \bar d_p} & 0_{\bar d_p \times \bar d_p} & 0_{\bar d_p \times \bar d_p} \\
            0_{\bar d_p \times d} & 0_{\bar d_p \times \bar d_p} & I_{\bar d_p \times \bar
            d_p} & 0_{\bar d_p \times \bar d_p} \\
            0_{\bar d_p \times d} & 0_{\bar d_p \times \bar d_p} & 0_{\bar d_p \times \bar d_p} & 0_{\bar d_p \times \bar d_p} 
    \end{bmatrix}
    \begin{bmatrix}
        2 x_{F_1} & 2 x_{F_2} & 2 x_{F_3} & F_1(x_d) \\
        p_{11} & p_{12} & p_{13} & p_{14} \\
        p_{21} & p_{22} & p_{23} & p_{24} \\
        p_{31} & p_{32} & p_{33} & p_{34}
    \end{bmatrix}\\
    &= P_2^T P_2
\end{align*}
and
\[ 
    K_2 Z_{B_1}  = \frac{1}{2} \begin{bmatrix}
            0_{d_v \times d_v} & 0_{d_f \times d_v} &  0_{d \times d_p}  \\
            0_{d_f \times d_v} & I_{d_f} &  0_{d \times d_p}  \\
            0_{d_v \times d} & 0_{d_f \times d} & 0_{d_p \times d_p} \\
        \end{bmatrix}
    \begin{bmatrix}
        2x_{F_1} & 2x_{F_2} & 2x_{F_3} & F_1(x_d) \\
        p_{1} & p_{2} & p_{3} & p_{4}
    \end{bmatrix}
    = 
    \begin{bmatrix}
        x_{F_1} & x_{F_2} & x_{F_3} & 0_{d} \\
        0_{d_p} & 0_{d_p} & 0_{d_p} & 0_{d_p} 
    \end{bmatrix}.
\]
Using the above, we can compute the output of the attention layer in the
second Transformer block which evaluates to
\begin{align*}
    \mathrm{Attn}_2(Z_{B_1}) + Z_{B_1}
    &=  (K_2 Z_{B_1}) (M \odot Z_{B_1}^T Q_2 Z_{B_1}) + Z_{B_1} \\
    &=  (K_2 Z_{B_1}) (M \odot P_2^T P_2) + Z_{B_1} \\
    &= 
    \begin{bmatrix}
        x_{F_1} & x_{F_2} & x_{F_3} & 0 \\
        0 & 0 & 0 & 0
    \end{bmatrix}
    \begin{bmatrix}
            1 & 0 & 0 & 0 \\
            0 & 1 & 0 & 1 \\
            0 & 0 & 1 & 0 \\
            0 & 0 & 0 & 1 
        \end{bmatrix} + Z_{B_1} \\
    &= \begin{bmatrix}
        3x_{F_1} & 3x_{F_2} & 3x_{F_3} & F_1(x_d) + x_{F_2}\\
        p_{1} & p_{2} & p_{3} & p_{4} \\
    \end{bmatrix}.
\end{align*}
The attention layer uses sub-matrix $P_2$ of the position encodings to copy the
second task token to the data token We repeat the calculations in~\Cref{eq:tr}, with $W_{21}$ and $W_{22}$ which yields
\begin{align*}
    \mathrm{Block}_2(\mathrm{Block}_1(Z))) 
    &= W_{22} \mathrm{ReLU}(W_{21} (\mathrm{Attn}_2(Z_{B_1}) + Z_{B_1})) + (\mathrm{Attn}_2(Z_{B_1}) + Z_{B_1}) \\
    &= \begin{bmatrix}
        3 x_{F_1} & 3 x_{F_2} & 3 x_{F_3} & F_2 \circ F_1(x_d) \\
        p_1 & p_2 & p_3 & p_4 \\
    \end{bmatrix}
    = Z_{B_2}.
\end{align*}

\paragraph{Step 3: Computing the output of the final Transformer block.}
Unsurprisingly, the calculations for the last Transformer block are almost
identical. The query matrix is $Z_{B_2}^T Q_3
Z_{B_2} = P_3^T P_3$ and the value matrix is
\[ 
    K_3 Z_{B_2}  = \frac{1}{3}
    \begin{bmatrix}
        x_{F_1} & x_{F_2} & x_{F_3} & 0_{d} \\
        0_{d_p} & 0_{d_p} & 0_{d_p} & 0_{d_p} 
    \end{bmatrix}
    \begin{bmatrix}
        3x_{F_1} & 3x_{F_2} & 3x_{F_3} & F_2 \circ F_1(x_d) \\
        p_{1} & p_{2} & p_{3} & p_{4}
    \end{bmatrix}
    = \begin{bmatrix}
        x_{F_1} & x_{F_2} & x_{F_3} & 0_{d} \\
        0_{d_p} & 0_{d_p} & 0_{d_p} & 0_{d_p} 
    \end{bmatrix}.
\]
The output of the attention layer in the final block is 
\begin{align*}
    \mathrm{Attn}_3(Z_{B_3}) + Z_{B_3}
    &=  (K_3 Z_{B_2}) (M \odot Z_{B_2}^T Q_2 Z_{B_2}) + Z_{B_2} \\
    &=  (K_3 Z_{B_1}) (M \odot P_3^T P_3) + Z_{B_2} \\
    &= 
    \begin{bmatrix}
        x_{F_1} & x_{F_2} & x_{F_3} & 0 \\
        0 & 0 & 0 & 0
    \end{bmatrix}
    \begin{bmatrix}
            1 & 0 & 0 & 0 \\
            0 & 1 & 0 & 0 \\
            0 & 0 & 1 & 1 \\
            0 & 0 & 0 & 1 
        \end{bmatrix} + Z_{B_2} \\
    &= \begin{bmatrix}
        4x_{F_1} & 4x_{F_2} & 4x_{F_3} & F_2 \circ F_1 (x_d) + x_{F_3}\\
        p_{1} & p_{2} & p_{3} & p_{4}
    \end{bmatrix}.
\end{align*}
Passing the output of $\mathrm{Attn}_2(Z_{B_2})$ through the last MLP, yields the output of the Transformer, which is
\begin{align*}
    \mathrm{Tr}(Z) 
    &=\mathrm{Block}_3(\mathrm{Block}_2(\mathrm{Block}_1(Z))) \\
    &= W_{32} \mathrm{ReLU}(W_{32} (\mathrm{Attn}_3(Z_{B_2}) + Z_{B_2})) + (\mathrm{Attn}_3(Z_{B_2}) + Z_{B_2}) \\
    &= \begin{bmatrix}
        4 x_{F_1} & 4 x_{F_2} & 4 x_{F_3} & F_3 \circ F_2 \circ F_1(x_d) \\
        p_1 & p_2 & p_3 & p_4
    \end{bmatrix}.
\end{align*}
Hence, the output of the Transformer is a composition of the three functions $F_1$, $F_2$ and $F_3$ applied to token $x_d$. \qed

\end{document}